\documentclass{article}

\usepackage{microtype}
\usepackage{graphicx}
\usepackage{booktabs} %

\usepackage{hyperref}

\usepackage[arxiv]{icml2023}

\usepackage{amsmath}
\usepackage{amssymb}
\usepackage{mathtools}
\usepackage{amsthm}

\usepackage{amsthm}
\usepackage{caption}
\usepackage{subcaption}

\usepackage{enumitem}

\definecolor{ari}{HTML}{1f77b4}
\definecolor{arp}{HTML}{ff7f0e}
\definecolor{arr}{HTML}{2ca02c}

\usepackage{listings}
\definecolor{codegreen}{rgb}{0,0.6,0}
\definecolor{codegray}{rgb}{0.5,0.5,0.5}
\definecolor{codepurple}{rgb}{0.58,0,0.82}
\definecolor{backcolour}{rgb}{0.95,0.95,0.92}
\lstdefinestyle{codestyle}{
    backgroundcolor=\color{backcolour},   
    commentstyle=\color{codegreen},
    keywordstyle=\color{magenta},
    numberstyle=\tiny\color{codegray},
    stringstyle=\color{codepurple},
    basicstyle=\ttfamily\footnotesize,
    breakatwhitespace=false,         
    breaklines=true,                 
    captionpos=b,                    
    keepspaces=true,                 
    numbers=left,                    
    numbersep=5pt,                  
    showspaces=false,                
    showstringspaces=false,
    showtabs=false,                  
    tabsize=2
}

\usepackage[nameinlink,capitalize,noabbrev]{cleveref}
\usepackage[capitalize,noabbrev]{cleveref}
\crefname{section}{Sec.}{Secs.}
\Crefname{section}{Section}{Sections}
\Crefname{table}{Table}{Tables}
\crefname{table}{Tab.}{Tabs.}
\Crefname{appendix}{App.}{App.}

\usepackage{crossreftools}

\usepackage{amsmath,amsfonts,bm}

\def\eqref#1{equation~\ref{#1}}

\def\1{\bm{1}}

\def\vs{{\bm{s}}}

\DeclareMathAlphabet{\mathsfit}{\encodingdefault}{\sfdefault}{m}{sl}
\SetMathAlphabet{\mathsfit}{bold}{\encodingdefault}{\sfdefault}{bx}{n}

\newcommand{\E}{\mathbb{E}}

\newcommand{\Ss}{\mathcal{S}}

\newcommand{\bX}{\mathbf{X}}
\newcommand{\bY}{\mathbf{Y}}

\newcommand{\precision}{\operatorname{RP}}
\newcommand{\recall}{\operatorname{RR}}
\newcommand{\aprecision}{\operatorname{ARP}}
\newcommand{\arecall}{\operatorname{ARR}}
\newcommand{\ari}{\operatorname{ARI}}
\newcommand{\fgaprecision}{\operatorname{FG-ARP}}
\newcommand{\fgarecall}{\operatorname{FG-ARR}}
\newcommand{\fgari}{\operatorname{FG-ARI}}

\newtheorem{prop}{Proposition}

\newtheorem{definition}{Definition}

\crefname{figure}{Fig.}{Figs.}
\crefname{prop}{Prop.}{Props.}
\crefname{appendix}{Appx.}{Appxs.}
\crefname{algorithm}{Alg.}{Algs.}
\crefname{theorem}{Thm.}{Thms.}
\crefname{equation}{Eq.}{Eqs.}
\crefname{definition}{Defn.}{Defns.}
\crefname{cor}{Corollary}{Corollaries}
\crefname{lem}{Lemma}{Lemmas}
\crefname{table}{Tab.}{Tabs.}

\newcommand{\ms}[1]{}
\newcommand{\zimmerrol}[1]{}
\newcommand{\svs}[1]{}
\newcommand{\tk}[1]{}

\renewcommand{\limits}{}

\usepackage{xspace}
\makeatletter
\DeclareRobustCommand\onedot{\futurelet\@let@token\@onedot}
\def\@onedot{\ifx\@let@token.\else.\null\fi\xspace}

\def\eg{\emph{e.g}\onedot} 
\def\ie{\emph{i.e}\onedot} 
\def\cf{\emph{cf}\onedot} 
 \def\vs{\emph{vs}\onedot}

\makeatother

\usepackage{tikz}
\newcommand*\circled[1]{\tikz[baseline=(char.base)]{
            \node[shape=circle,draw,inner sep=2pt,scale=0.8] (char) {#1};}}

\newcommand{\todo}[1]{\textcolor{red}{[\textbf{TODO} #1]}}
\newcommand\TODO\todo

\icmltitlerunning{Sensitivity of Slot-Based Object-Centric Models to their Number of Slots}

\begin{document}

\twocolumn[
\icmltitle{Sensitivity of Slot-Based Object-Centric Models to their Number of Slots}

\icmlsetsymbol{equal}{*}

\begin{icmlauthorlist}
\icmlauthor{Roland S. Zimmermann}{equal,tue}
\icmlauthor{Sjoerd van Steenkiste}{goog}
\icmlauthor{Mehdi S. M. Sajjadi}{gdm}
\icmlauthor{Thomas Kipf}{gdm}
\icmlauthor{Klaus Greff}{gdm}
\end{icmlauthorlist}

\icmlaffiliation{goog}{Google Research} 
\icmlaffiliation{gdm}{Google DeepMind}  %
\icmlaffiliation{tue}{T\"ubingen AI Center, Max Planck Institute for Intelligent Systems}

\icmlcorrespondingauthor{Roland S. Zimmerman}{research@rzimmermann.com}

\icmlkeywords{Object-centric learning, Representation learning, Perceptual Grouping}

\vskip 0.3in
]

\printAffiliationsAndNotice{\textsuperscript{*}Work done while at Google Research.}  %

\begin{abstract}
    Self-supervised methods for learning object-centric representations have recently been applied successfully to various datasets.
    This progress is largely fueled by slot-based methods, whose ability to cluster visual scenes into meaningful objects holds great promise for compositional generalization and downstream learning.
    In these methods, the number of slots (clusters) $K$ is typically chosen to match the number of ground-truth objects in the data, even though this quantity is unknown in real-world settings.
    Indeed, the sensitivity of slot-based methods to $K$, and how this affects their learned correspondence to objects in the data has largely been ignored in the literature.
    In this work, we address this issue through a systematic study of slot-based methods. 
    We propose using analogs to precision and recall based on the Adjusted Rand Index to accurately quantify model behavior over a large range of $K$.
    We find that, especially during training, incorrect choices of $K$ do not yield the desired object decomposition and, in fact, cause substantial oversegmentation or merging of separate objects (undersegmentation).
    We demonstrate that the choice of the objective function and incorporating instance-level annotations can moderately mitigate this behavior
    while still falling short of fully resolving this issue.
    Indeed, we show how this issue persists across multiple methods and datasets and stress its importance for future slot-based models.
\end{abstract}

\section{Introduction}

 \begin{figure}[tbh]
        \centering
        \includegraphics[width=\linewidth]{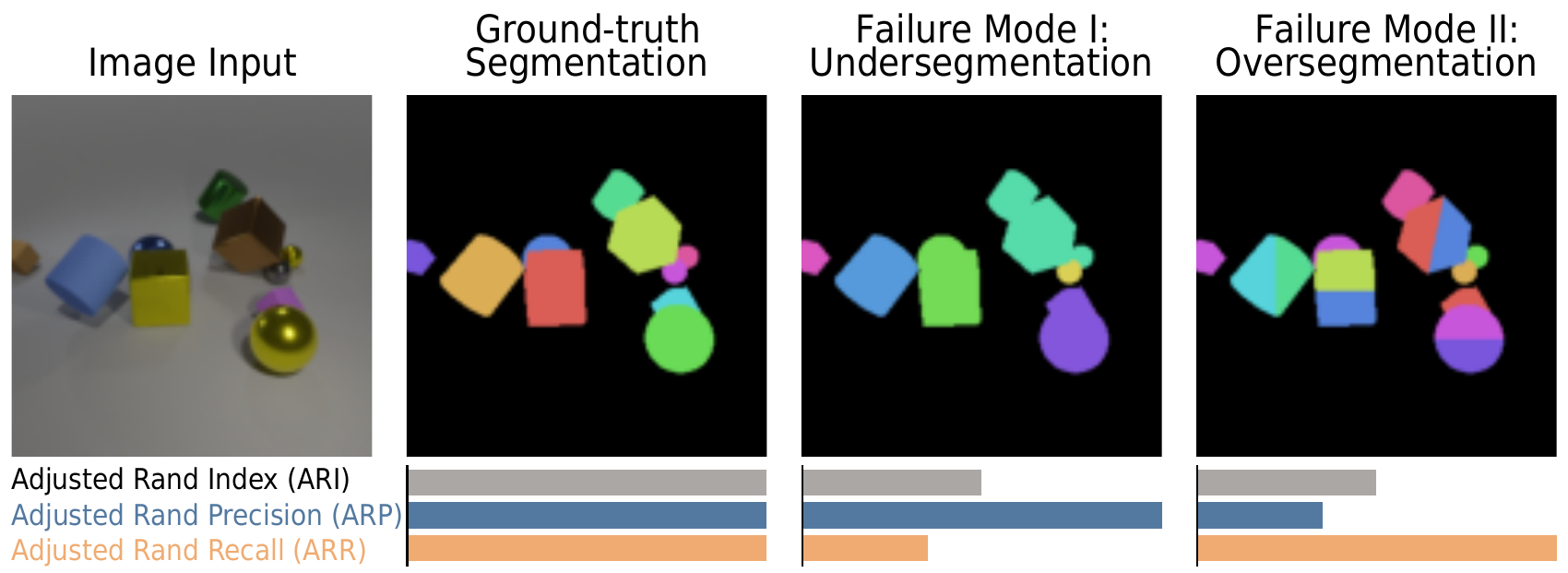}
        \caption{
        \textbf{Differentiating distinct failure modes.}
        We show that certain design choices influence the learned notion of objectness in object-centric models and identify two common failure modes: (i) the model uses a too fine-grained definition of objects, resulting in split-up objects (oversegmentation) and (ii) the model's definition of objects is too coarse-grained, leading to merging of objects (undersegmentation). Although these failure modes are contrary, it is impossible to distinguish them with the common $\ari$ metric, whereas our newly proposed $\aprecision$ and $\arecall$ metrics clearly separate them (see also \cref{fig:precsion_recall_ari_checkerboard}).}
        \label{fig:over_undersegmentation}
    \end{figure}

    Humans perceive and understand the world in terms of objects and their relationships. 
    By gluing together raw perception and symbol-like abstraction, objects form the fundamental building blocks of our higher level cognition, and support many of our impressive generalization capabilities \citep{spelke2007core, johnson2010mental}.
    The goal of object-centric representation learning is to replicate this ability and, thus, afford neural networks a similar robustness and capacity for systematic generalization. 
    Self-supervised learning of meaningful object representations is a challenging problem, and most work in this area is limited to well-controlled synthetic datasets, though there has recently been significant progress, with datasets becoming progressively more complex and realistic \cite{karazijaClevrTexTextureRichBenchmark2021, greffKubricScalableDataset2022, sajjadiObjectSceneRepresentation2022, elsayedSAViEndtoEndObjectCentric2022a}.
    This progress is driven largely by neural slot-based approaches which learn to discover meaningful objects by iteratively clustering their inputs into a set of slots \citep{greffBindingProblemArtificial2020}. In recent years, most slot-based approaches have been powered by Slot Attention \citep{locatelloObjectCentricLearningSlot2020}.
    
    In said methods, the number of slots $K$ is a hyperparameter that is typically assumed fixed and chosen a priori to match the intended (ground-truth) number of objects in the data.
    Unfortunately, for real-world settings, the ``true number'' of objects is usually unknown and can differ by orders of magnitude depending on the scene and task.  
    More fundamentally, even the definition of what constitutes an object becomes less clear here and it is up to a certain degree task-dependent: Is a tree a single object or should every branch and leaf be considered as a separate object?
    To scale slot-based methods to real-world data then, it is important to understand the effect of the hyperparameter $K$ on the learned object representations and their correspondence to the intended objects in the scene.
    How sensitive are slot-based methods to the choice of $K$, and what happens if it is chosen too small or too large? 
    The number of slots can often also be varied after training, which raises an additional question about its effect at inference time. 
    Finally, what is the effect of common variations such as changes in architecture, slot initialization, and training objective?

    In this paper, we present the first systematic study that investigates these questions in detail.
    In particular, we make the following contributions:
    \begin{itemize}
        \item We propose analogs to precision and recall based on the Adjusted Rand Index (ARI) \citep{randObjectiveCriteriaEvaluation1971} to quantify the extent to which models are oversegmenting \vs undersegmenting. 
        \item We empirically evaluate three recent Slot Attention based methods on five different datasets and provide insights into their behavior for large ranges of $K$ both during training and inference time.
        \item Finally, we investigate the role of different learning objectives and slot initializations on this phenomenon. 
    \end{itemize}

\section{Background} \label{sec:related_work}
\paragraph{Object-Centric Representation Learning}
    in an unsupervised fashion has been an active area of research for the past years \cite{greffBindingProblemArtificial2020, yuanCompositionalSceneRepresentation2022a}.
    Although models for both images and videos have been scaled from simple toy data \citep{greffBindingReconstructionClustering2016} to more complex datasets \citep{kipfConditionalObjectCentricLearning2022, elsayedSAViEndtoEndObjectCentric2022a}, there still remains a gap to real-world data. 
    Most approaches for object-centric learning are slot-based, meaning that they aim to extract and represent distinct objects from the data into separate variables (called slots) \citep{greffBindingProblemArtificial2020}. Here, the maximum number of slots $K$ is a hyperparameter that needs to be tuned. There is work proposing more automated ways to infer the number of slots, however, this comes with the cost of introducing other hyperparameters and shifting the problem to a different stage of the inference \citep{engelckeGENESISV2InferringUnordered2022, bearLearningPhysicalGraph2020}.

    In this work, we focus on analyzing three methods in particular.
    Slot Attention (SA) \citep{locatelloObjectCentricLearningSlot2020} is an algorithm that iteratively extracts information from input data and stores them into object slots using attention \citep{luongEffectiveApproachesAttentionbased2015, bahdanauNeuralMachineTranslation2016, vaswaniAttentionAllYou2017}. This means that slots compete for representing information, based on the information they already contain. Effectively, SA can be seen as a k-means clustering \citep{MacQueen1967} applied on learned features.
    Slot Attention for Video (SAVi) \citep{kipfConditionalObjectCentricLearning2022} extends SA from static images to videos by introducing a transition network modeling the temporal dynamics of each slot. Furthermore, this architecture can use additional signals to condition the model on certain objects, \eg, by using the bounding boxes of objects in the first frame of a video.
    The Object Scene Representation Transformer (OSRT) \citep{sajjadiObjectSceneRepresentation2022} combines SA with SRT~\citep{sajjadiSceneRepresentationTransformer2022}, learning a 3D-aware object-centric representation from multiple views. After training, the model can synthesize novel views and their corresponding object segmentation masks.
    
    \paragraph{Model Evaluation}
    and comparison between object-centric models depends on the dataset and architecture: If the full information used to generate the dataset is available, readout performance for different object-specific properties can be used to quantify the representations' quality \citep{locatelloObjectCentricLearningSlot2020, dittadiGeneralizationRobustnessImplications2022}.    
    However, if this information is not available, as is the case for most real-world datasets, a different approach is needed.
    Here, a typical approach is to take segmentation maps extracted from the model\footnote{For example, in case of generative compositional models \citep[\eg][]{burgessMONetUnsupervisedScene2019} the compositing mask, and in case of attention-based methods the encoder's attention map, can be interpreted as such a segmentation map.} and compare these with ground-truth instance-level segmentation maps.
    
    In the past, different metrics have been used for the evaluation of segmentation masks (produced by object-centric models). For one, the Adjusted Mutual Information (AMI) \citep{vinhInformationTheoreticMeasures2010} has been used to evaluate segmentation masks of object-centric models. For another, the mean Intersection-over-Union (mIoU) \citep{jaccardDistributionFloreAlpine1901, engelckeGENESISGenerativeScene2020}, which includes solving a matching problem between predictions and ground truth, %
    and the mean Segmentation Coverage (mSC) \citep{arbelaezContourDetectionHierarchical2011, engelckeGENESISGenerativeScene2020} have been proposed.

    Finally, the go-to choice of metric in recent years is the Adjusted Rand Index (ARI) \citep{randObjectiveCriteriaEvaluation1971, hubertComparingPartitions1985, greffMultiObjectRepresentationLearning2020a} which treats the segmentation problem as a clustering task and measures clustering similarity such that it is invariant under arbitrary permutations of the clusters.
    For two segmentation maps $\bX, \bY \in \mathbb{Z}^N$ with up to $I$ and $J$ classes, respectively, the ARI is defined as \cite{randObjectiveCriteriaEvaluation1971, albatineh2006similarity}:
    \begin{align*}
        \ari(\bX, \bY) &= \frac{\sum\limits_{i=1}^I\sum\limits_{j=1}^J m_{ij}^2 - \E_{\bY'}\left[\sum\limits _{i=1}^I\sum\limits _{j=1}^J m_{ij}^2\right]}{P + Q + 2m - \E_{\bY'}\left[\sum\limits_{i=1}^I\sum\limits_{j=1}^J m_{ij}^2\right]}.
    \end{align*}
    Here, $m_{ij}$ denotes the matching matrix indicating how many pixels are segmented as label $i$ and $j$ in $\bX$ and $\bY$ respectively. Further $m = \sum_{ij} m_{ij}$ is the total number of pixels, $Q = \sum_{i=1}^I m_{i+}^2 - m$ and $P = \sum_{j=1}^J m_{+j}^2 - m$. The expectation value used here is computed using a hypergeometric distribution \citep{hubertComparingPartitions1985}.
    
    While a perfect ARI score is indicative of a model whose notion of objects is well aligned with the target/human notion, interpreting subpar scores is less clear. In particular, two \emph{differently behaving models} -- one that learned a too coarse notion of objects and merges independent ground-truth objects, and one that learned a too fine-grained notion and splits up ground-truth objects into multiple objects -- can yield the \emph{same} ARI score while clearly exhibiting different failure modes (for example see \cref{fig:precsion_recall_ari_checkerboard}). Note that the other previously mentioned metrics share this shortcoming.

    Further evaluation metrics were proposed in the domain of image segmentation. \citet{gong2011conditional} introduce two metrics based on conditional entropies to detect over- and undersegmentation. Further, precision and recall scores based on the boundary contours of segmentation maps were proposed \citep{martinLearningDetectNatural2004}. However, they have not been applied for the evaluation of object-centric models yet, and their value scales are incompatible with that of the commonly used ARI score, complicating comparisons.
    
    Although the mIoU metric appears similar to the (non-adjusted) Rand Recall, note that it grants fewer insights into a model's behavior as it does not distinguish between too large and too small segmentations. A similar argument can also be made for the Mean Average Precision (mAP) metric \citep{hariharanSimultaneousDetectionSegmentation2014} and the proposed ARP, as the mAP metric leverages the IoU metric. 

\section{Beyond ARI with Precision and Recall}

    To obtain a more complete picture of the behavior and performance of slot-based models, we propose to use an additional set of metrics. Among other things, these allow us to \emph{distinguish models that over- and undersegment a scene} and, therefore, provide more fine-grained insights into a model's behavior. As outlined in \cref{sec:related_work} and visible in \cref{fig:over_undersegmentation} the commonly used $\ari$ score groups models with different behaviors resulting in a less detailed understanding and detection of potential shortcomings of models. 
    
    Inspired by the notion of precision and recall from the information retrieval literature \citep{rijsbergenInformationRetrieval1979} we introduce two extensions %
    of the ARI metric for the evaluation of object-centric models:
    The Adjusted Rand Precision (ARP) measures how many pairs of pixels that are grouped together in the model's prediction belong to the same object in the ground truth; conversely, the Adjusted Rand Recall (ARR) measures how many pixel pairs of the same ground-truth object are grouped together in the predictions. Here, the term \emph{adjusted} refers to adjusting the metrics for chance agreement by normalizing them with the value expected for randomly shuffled segmentation maps.  Due to the similarity with the $\ari$ metric, resulting in the same value range, interpreting results remains relatively intuitive and $\ari$, $\aprecision$ and $\arecall$ can easily be compared.

    Specifically, we define the ARP and ARR as:
    \begin{align*}
        \aprecision(\bX, \bY) &= \frac{\sum\limits _{i=1}^I\sum\limits_{j=1}^J m_{ij}^2 - \E_{\bY'}\left[\sum\limits _{i=1}^I\sum\limits _{j=1}^J m_{ij}^2\right]}{Q + m - \E_{\bY'}\left[\sum\limits _{i=1}^I\sum\limits _{j=1}^J m_{ij}^2\right]}
    \end{align*}
    \begin{align*}
        \arecall(\bX, \bY) &= \frac{\sum\limits_{i=1}^I\sum_{j=1}^J m_{ij}^2 - \E_{\bY'}\left[\sum\limits _{i=1}^I\sum\limits _{j=1}^J m_{ij}^2\right]}{P + m - \E_{\bY'}\left[\sum\limits _{i=1}^I\sum\limits _{j=1}^J m_{ij}^2\right]},
    \end{align*}
    where $m_{ij}, m, P, Q$ are defined as in \cref{sec:related_work}. While the precision $\aprecision(\bX, \bY)$ measures the fraction of pixel pairs belonging to the same segment in $\bY$ given that they do in $\bY$, the recall $\arecall(\bX, \bY)$ measures opposite, \ie, the faction of pixel pairs belonging to the same segment in $\bX$ given that they do in $\bX$. Note, that $\aprecision$ and $\arecall$ are antisymmetric with respect to their arguments, \ie $\aprecision(\bX, \bY) = \arecall(\bY, \bX)$ (see \cref{prop:ar_ap_antisymmetry}).

    An illustration of the abilities of these metrics to quantify a model's behavior is given in \cref{fig:precsion_recall_ari_checkerboard}. Specifically, we compare the behavior of $\ari$, $\aprecision$ and $\arecall$ for different failure cases of models showing that $\aprecision$ and $\arecall$ yield insights into model failures indistinguishable by $\ari$.
    
    \begin{figure}[tb]
        \centering
        \includegraphics[width=\linewidth]{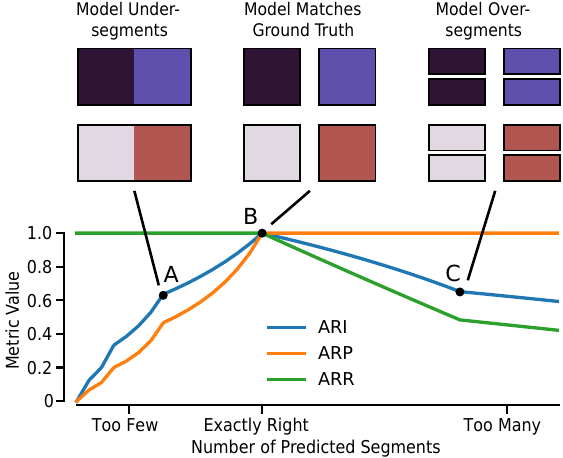}
        \caption{\textbf{Evaluation metrics on toy data.}
        The top shows three segmentations with different characteristics: A) undersegmentation, B) perfect segmentation, and C) oversegmentation.
        While ARI is sensitive to both failures, it does not distinguish between A and C which have identical ARI scores.
        In contrast, ARP and ARR allow one to easily distinguish these distinct failure cases.
        More details in \cref{sec:precsion_recall_ari_checkerboard_details}.
        }
        \label{fig:precsion_recall_ari_checkerboard}
    \end{figure}
    
    We supply an implementation of these metrics in \cref{lst:arr_arp}, while a more detailed definition and analysis of these metrics are shown in \cref{appx:definition_properties_precision_recall}. Most importantly, the $\ari$ score is always bound to be between $\aprecision$ and $\arecall$ (\cref{prop:ap_ar_bound_ari}) and the $\ari$ can be seen as an $F_1$ score of $\aprecision$ and $\arecall$ (\cref{prop:ari_f1_harmonic_mean_ap_ar}).
    On a high level, the proposed $\aprecision$ metric can only be high if the model does not oversegment the input (\ie, if it's notion of objectness is not too fine), while the $\arecall$ can only be high if the model does not undersegment (\ie if its notion of objectness is not too coarse).
    As we are only interested in measuring how well the model discovers and segments objects, we follow common practice \citep{greffTaggerDeepUnsupervised2016}
    and ignore non-object pixels when computing the metrics; we denote these variants of the metrics as FG-ARI (Foreground ARI), FG-ARP and FG-ARR.

    Although these metrics were not originally proposed by \citet{randObjectiveCriteriaEvaluation1971}, we chose their name to highlight the similarity to the ARI -- a metric computed over the entire segmentation map -- and at the same time distinguish them from the sometimes used definition of precision/recall leveraging segmentation contours \citep{arbelaezContourDetectionHierarchical2011, martinLearningDetectNatural2004}. We also note that a set of related metrics was previously mentioned in a different context, namely for the evaluation of clustering results \citep{wallaceMethodComparingTwo1983} and medical imaging \citep{arganda-carrerasCrowdsourcingCreationImage2015}.

\section{Experiments} \label{sec:experiments}
    An overview of the models and datasets analyzed in this work is displayed in \cref{fig:models}. Technical details on the architectures used can be found in \cref{sec:experiemntal_details}. 

\begin{figure}
    \centering
    \includegraphics[width=\columnwidth]{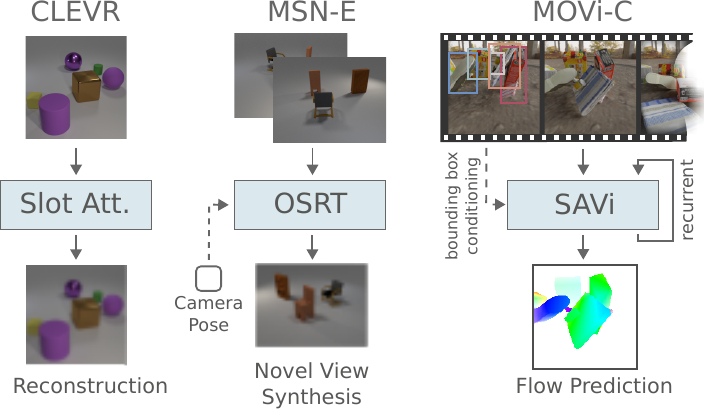}
    \caption{\textbf{Overview of analyzed models and datasets.} We empirically investigate the sensitivity of three slot-based models to the chosen number of slots, namely a vanilla Slot Attention architecture, SAVi and OSRT. These models are trained and evaluated on image datasets --- \ie, CLEVR \citep{johnsonCLEVRDiagnosticDataset2016}, MSN \cite{stelznerDecomposing3DScenes2021} -- and video datasets --- \ie,  MOVi-A/C \cite{greffKubricScalableDataset2022}, CATER \cite{girdharCATERDiagnosticDataset2020}.}
    \label{fig:models}
\end{figure}

\subsection{Sensitivity to the Number of Slots}
\label{sec:number_of_slots_clevr_cater}

    \newcommand\circlenr[1]{
    \textcircled{\raisebox{-0.9pt}{#1}}
    }
    \newcommand\qpicgt[2]{
        \begin{minipage}{0.0875\linewidth}
            \vspace{0.5em}
            \begin{center}\footnotesize#2\end{center}
            \vspace{-0.45em}
            \includegraphics[width=\linewidth]{figures/qualitative/hard/gt_#1}
        \end{minipage}
    }
    \newcommand\qpic[2]{
        \begin{minipage}{0.0875\linewidth}
            \vspace{0.5em}
            \begin{center}\circlenr{\small{#2}}\end{center}
            \vspace{-0.45em}
            \includegraphics[width=\linewidth]{figures/qualitative/hard/#1}
        \end{minipage}
    }

    \begin{figure*}[tbh]
        \centering
        \includegraphics[width=\textwidth]{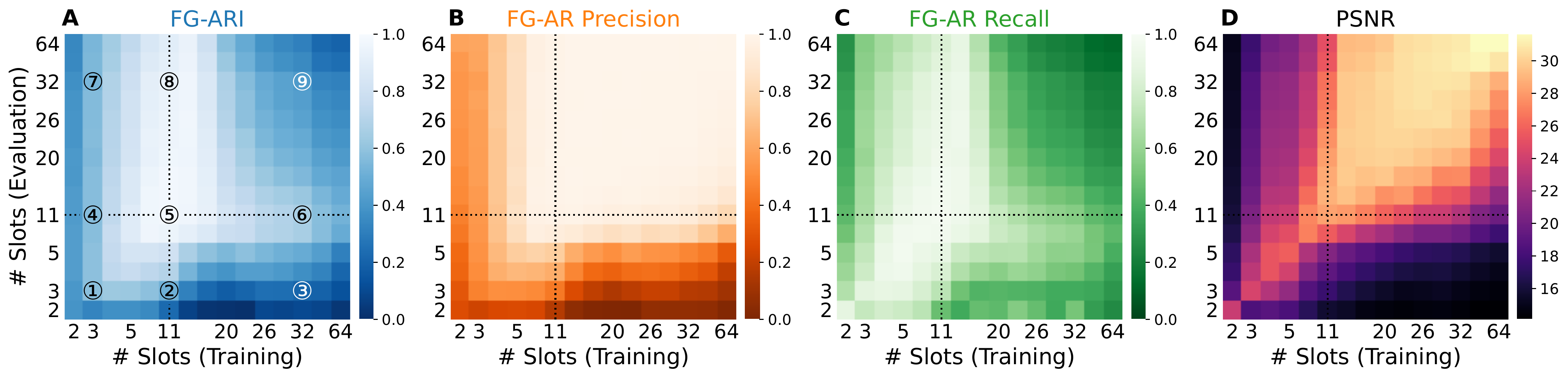}
        \setlength{\tabcolsep}{-.2mm}
        \def\arraystretch{2}
        \begin{tabular}{cc@{\hskip 0.5mm}ccccccccc}
            \qpicgt{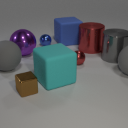}{GT \vphantom{g}RGB} &
            \qpicgt{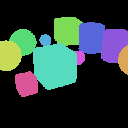}{GT Segm.} &
            \qpic{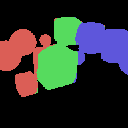}{1} &
            \qpic{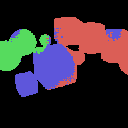}{2} &
            \qpic{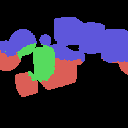}{3} &
            \qpic{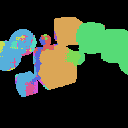}{4} &
            \qpic{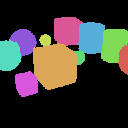}{5} &
            \qpic{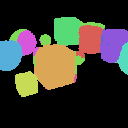}{6} &
            \qpic{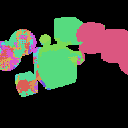}{7} &
            \qpic{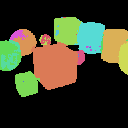}{8} &
            \qpic{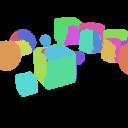}{9}
        \end{tabular}
        \caption{\textbf{Sensitivity of Slot Attention to the number of slots on CLEVR.} We analyze the sensitivity of Slot Attention during training and inference, measured by Foreground Adjusted Rand Index ($\fgari$, \textbf{A}), Foreground Adjusted Rand Precision ($\fgaprecision$, \textbf{B}), Foreground Adjusted Rand Recall ($\fgarecall$, \textbf{C}), and reconstruction quality (PSNR, \textbf{D}). See \cref{fig:precsion_recall_ari_checkerboard} for a description of $\aprecision$ and $\arecall$. The black dotted lines indicate that the minimally required number of slots was either used during training (vertical, analyzed in \cref{sec:number_of_slots_clevr_cater_inference}) or during inference (horizontal, analyzed in \cref{sec:number_of_slots_clevr_cater_training}).
        The bottom row shows the predicted segmentations for nine different combinations of $K_{\textrm{train}}$ and $K_{\textrm{eval}}$ as indicated in \textbf{A}. 
        }
        \label{fig:clevr_sa_n_slots_train_eval}
    \end{figure*}

    We study the robustness of slot-based models towards domain mismatch/misspecification.
    For a fixed dataset, we ask: How does the behavior of a model change when its number of slots $K$ is varied? How should $K$ be chosen if the true number of objects during training or inference is unknown? And can simple unsupervised metrics such as the reconstruction error (\ie PSNR) for auto-encoding models inform this decision?
    These are informative questions to be able to apply unsupervised slot-based object-centric models to real-world data as the number of objects is unknown and might vary substantially across scenes.
    
    To investigate this type of robustness, we begin by training and evaluating 225 Slot Attention configurations ~\citep{locatelloObjectCentricLearningSlot2020} that vary the number of slots during training and inference on the CLEVR dataset \citep{johnsonCLEVRDiagnosticDataset2016}.
    For each configuration, we train three models with different seeds, and measure the $\ari$, $\aprecision$,  $\arecall$ and PSNR, which we report in \cref{fig:clevr_sa_n_slots_train_eval}.
    To be able to compare, we also train a large number of other slot-based models, including OSRT~\citep{sajjadiObjectSceneRepresentation2022} and SAVi~\citep{kipfConditionalObjectCentricLearning2022}, and train on other datasets such as MSN \citep{stelznerDecomposing3DScenes2021}, CATER \citep{girdharCATERDiagnosticDataset2020} and MOVi-A/C \citep{greffKubricScalableDataset2022}. These results are summarized in \cref{fig:clevr_msn_e_cater_movi_sa_savi_osrt_n_slots}.

\paragraph{During inference}
\label{sec:number_of_slots_clevr_cater_inference}
    Focusing on the effect of changing the number of slots during inference in \cref{fig:clevr_sa_n_slots_train_eval} (vertical black dotted line), it can be seen that when Slot Attention is trained with a certain number of slots $K_{\mathrm{train}}$ that is close the to minimally required number of slots to represent all objects (\ie $K_{\mathrm{opt}} \approx 11$), increasing the number of slots $K_{\mathrm{eval}}>K_{\mathrm{train}}$ during inference has little effect on both the $\ari$ and $\arecall$ and no effect at all on the $\aprecision$ (compare also \circled{8} to \circled{5} empirically).
    However, it can also be seen that decreasing the number of slots $K_{\mathrm{eval}} < K_{\mathrm{train}}$ mostly reduces $\aprecision$ and thereby also the $\ari$ score \circled{2}.
    Similarly, there is a substantial decrease in PSNR in this case.
    Taken together, this suggests a viable strategy for applying object-centric models to real-world data with unknown number of objects: Train the model on a curated dataset for which an upper bound on the number of objects is known and give the model access to substantially more slots at inference time, \eg, as many slots as objects might show up in the scene.
    Indeed, it was previously shown that slot-based models are able to generalize to scenes having additional \emph{objects} at inference time when enough slots are available, \eg \citet{greffMultiObjectRepresentationLearning2020a, dittadiGeneralizationRobustnessImplications2022}.

\paragraph{During training} \label{sec:number_of_slots_clevr_cater_training}
    If we consider Slot Attentions' behavior when varying the number of slots only during training, then it appears more sensitive to the specific choice of the number of slots $K_{\mathrm{train}}$ in \cref{fig:clevr_sa_n_slots_train_eval} (horizontal black dotted line). If the model is trained with substantially too few or many slots (\ie $K_{\mathrm{train}} \ll K_{\mathrm{opt}} \lor K_{\mathrm{train}} \gg K_{\mathrm{opt}}$) we observe low $\ari$ scores, indicating that the learned object decomposition is not aligned with the ground truth. If $K_{\mathrm{train}}$ is set too high the model yields a poor $\arecall$ while maintaining a high $\aprecision$, indicating that severe oversegmentation occurs \circled{6}.
    On the other end, if $K_{\mathrm{train}}$ is too low both the $\aprecision$ and $\arecall$ decrease, hinting at a combined phenomenon where the model merges some objects but splits others up too \circled{4}. Importantly, after training using a poor estimate of $K_{\mathrm{train}}$, the model's performance can not be recovered at inference time through a better estimate for $K_{\mathrm{eval}}$ (compare to \circled{5}).

    \begin{figure*}[tb]
        \centering
         \includegraphics[width=\textwidth]{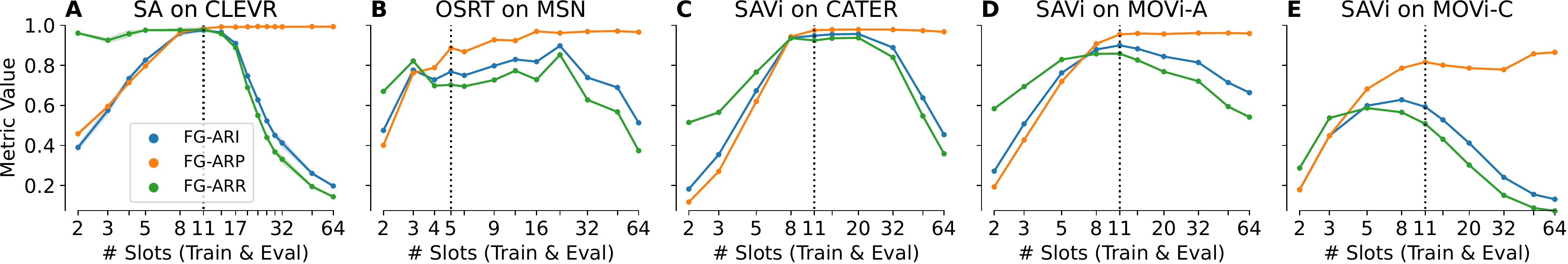}
         \caption{
         \textbf{Sensitivity of Slot-based Models on Various Datasets.}
          Characterization of the sensitivity of Slot Attention (SA) on CLEVR (\textbf{A}), the Object Scene Representation Transformer (OSRT) on MSN (\textbf{B}) and Slot Attention for Videos (SAVi) on CATER (\textbf{C}), MOVi-A/C (\textbf{D}, \textbf{E}) through $\fgari$ (\textcolor{ari}{blue}), $\fgaprecision$ (\textcolor{arp}{orange}) and $\fgarecall$ (\textcolor{arr}{green}). We see the same pattern for all models and datasets, namely, that there exists an intermediate region with all high $\ari$, $\aprecision$ and $\arecall$ values whereas too few or too many slots lead to a drop in either  $\arecall$ or $\aprecision$. All models used here use a random slot initialization. The dotted line at eleven (five for MSN) highlights the minimally required number of slots to represent the scene's background and the maximum number of possibly occurring objects.
         Note that the first plot (SA on CLEVR) corresponds to the bottom-left to top-right diagonal in \cref{fig:clevr_sa_n_slots_train_eval}.}
         \label{fig:clevr_msn_e_cater_movi_sa_savi_osrt_n_slots}
    \end{figure*}
    
\paragraph{During training and inference} \label{sec:number_of_slots_clevr_cater_training_inference}
    The models in the previous two paragraphs were tested differently from how they had been trained, and their behavior might be explained by this domain mismatch. Here we consider the effect of simultaneously changing the number of slots during training and inference time using $K_{\mathrm{train}} = K_{\mathrm{eval}}$.
   In addition to the Slot Attention models on CLEVR, we now also consider a variety of slot-based models --- Slot Attention, OSRT and SAVi --- on multiple datasets --- CLEVR, MSN, CATER and MOVi-A/C. %
    In \cref{fig:clevr_msn_e_cater_movi_sa_savi_osrt_n_slots} it can be seen how these models behave nearly identically across these datasets, and, in fact, resemble the behavior of the toy models performing over- and undersegmentation in \cref{fig:precsion_recall_ari_checkerboard}.
    Moreover, comparing \cref{fig:clevr_msn_e_cater_movi_sa_savi_osrt_n_slots}A to the horizontal black dotted line in \cref{fig:clevr_sa_n_slots_train_eval} we observe many similarities, suggesting that the number of slots used during training contributes more to the observed behavior.
    We further analyze the dependence of model performance on the number of ground-truth objects in \cref{sec:model_performance_number_of_objects}.

    When there are fewer slots than there are possible objects in the scenes (regime left of the vertical black dotted line in \cref{fig:clevr_msn_e_cater_movi_sa_savi_osrt_n_slots}), there does not exist a perfect solution anymore and the model has to start representing independent objects in shared slots. It can be seen that the $\ari$, $\aprecision$, and $\arecall$ all correlate with the number of slots.
    As both the $\aprecision$ and the $\arecall$ are typically low for models having too few slots, they are neither only performing over- nor undersegmentation (unlike on CLEVR, \cf \circled{1}). %
    Thus, we know that model is typically not just putting objects into shared slots, but also splits ground-truth objects up into parts and places (parts of) different objects in shared slots.
    Note how this distinction is not clear from only looking at the PSNR in \cref{fig:clevr_sa_n_slots_train_eval}.

    On the other hand, when the model has more slots than it potentially requires to represent all objects in a scene, multiple conceivable behaviors are possible:
    1) it could simply leave the additional slots unused, 2) use additional slots to split ground-truth objects into additional parts, either aligned with a human notion of objectness operating at a higher granularity level or arbitrary parts, or, 3) not bind slots to individual objects anymore but instead default to an (input independent) tessellation solution, \ie, it divides the image into (seemingly random) patches.
    From \cref{fig:clevr_msn_e_cater_movi_sa_savi_osrt_n_slots} (regime right of vertical black dotted line) we observe that the $\aprecision$ remains relatively constant at a near perfect score as we further increase the number of slots.
    At the same time, both the $\ari$ and $\arecall$ now decrease with an increasing number of slots, which, therefore, tells us that the model oversegments the image and splits ground-truth objects up into multiple parts \ie, scenario 2) (see also \circled{9}), which can not be concluded from looking at $\ari$ or PSNR alone. Note that while the model does oversegment, it does not appear to be splitting objects into arbitrary segments, but rather parts that are (partially) aligned with human notion of objectness (\eg, the sides of a cube are split into separate objects). 

    \paragraph{Correlation with PSNR}
    Returning to \cref{fig:clevr_sa_n_slots_train_eval}, it can be seen how the reconstruction quality (PNSR, \cref{fig:clevr_sa_n_slots_train_eval}D) is not strongly correlated with $\ari$ but rather with $\aprecision$ (Pearson correlation of $0.45$ \vs $0.88$).
    We also note that the PSNR appears symmetric (Pearson correlation with its transpose of $0.83$), meaning the number of slots during training and inference have similar effects on the reconstruction quality.
    Indeed, it is clear that PSNR is a poor criterion for model selection in slot-based models with regard to their ability to discover meaningful objects.

\subsection{Role of Training Objective}
    \begin{figure*}[tb]
        \centering
        \includegraphics[width=\textwidth]{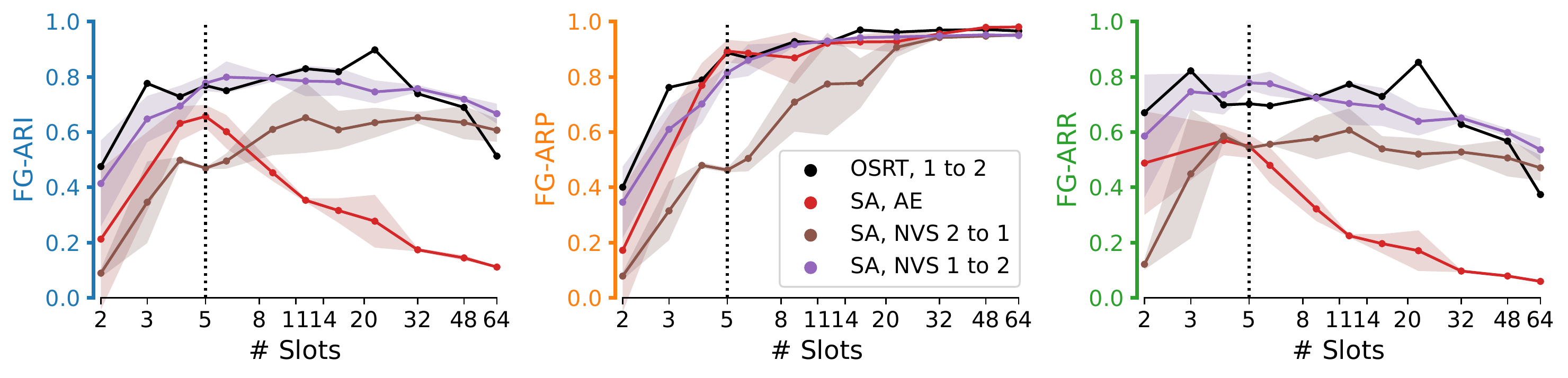}
        \caption{\textbf{SA and OSRT on MSN}. Comparison of the sensitivity of different slot-based architectures and training objectives to the number of slots: Object Scene Representation Transformer (OSRT), Slot Attention (SA) with an auto-encoder (AE) and novel-view-synthesis (NVS) objective. The suffix $k$ to $n$ means that the model predicts $n$ target from $k$ input views.
        }
        \label{fig:msn_e_osrt_sa_ae_nvs}
    \end{figure*}
    
    The analysis in \cref{sec:number_of_slots_clevr_cater_inference} revealed that there is less of a performance drop-off when increasing number of slots using $K_{\mathrm{train}} = K_{\mathrm{eval}}$ for OSRT compared to Slot Attention (compare \cref{fig:clevr_msn_e_cater_movi_sa_savi_osrt_n_slots}A \& B).  
    One key distinguishing factor between Slot Attention and OSRT is the novel-view synthesis objective used in the latter (as opposed to auto-encoding), whereby the model is trained to synthesize novel views given new camera poses.

    We investigate the role of novel-view synthesis further using the MultiShapeNet (MSN) dataset \citep{stelznerDecomposing3DScenes2021}, which contains scenes with three distinct views per scene for training purposes. We compare four types of models: the originally proposed OSRT model predicting two target views from one input view (OSRT, 1 to 2); Slot Attention with an auto-encoding objective on all three views (SA, AE); Slot Attention with a novel-view synthesis objective predicting one target from two input views (SA, NVS 2 to 1); and Slot Attention with a novel-view synthesis objective predicting two targets from a single input view (SA, NVS 1 to 2). The last model (SA, NVS 1 to 2) differs from the first model (OSRT, 1 to 2) only in terms of the model architecture. See \cref{sec:experiemntal_details} for details on the architectures used.
    
    As before, we measure $\ari$, $\aprecision$ and $\arecall$ while varying the number of slots using $K_{\mathrm{train}} = K_{\mathrm{eval}}$, which we report in \cref{fig:msn_e_osrt_sa_ae_nvs}.
    For the Slot Attention model with the AE objective we see the same trend as in \cref{fig:clevr_sa_n_slots_train_eval} for the CLEVR dataset: $\ari$ and $\arecall$ peak at the optimal number of slots and decline for too many slots, while $\aprecision$ increases with an increasing number of slots.
    Interestingly, however, the three models that are trained with the NVS objective are substantially less sensitive to the chosen number of model slots, regardless of the exact model architecture and the number of target or input views. While it is still true that the $\arecall$ decreases with an increasing number of slots, the decline is much less severe and only becomes noticeable at an extreme divergence between the chosen and optimal numbers of slots (48 \vs 5). Further, as we see almost the same performance for any of the models trained with NVS, we conclude that architectural nuances or the number of input/target views for the NVS objective are less important than the inductive bias induced by the objective itself.
    
    This supports our initial observation that a different training objective can induce a different object-binding behavior.
    While it is encouraging to see that this objective can increase stability across a larger range of hyperparameters, it is surprising that it appears to adopt the specific notion of objectness (and thus level of granularity) used in the ground truth, and is not inclined to use a more fine-grained object segmentation.%
    We hypothesize that this behavior is due to a stronger inductive bias by the NVS objective bias against oversegmentation as parts of objects are never independently varied, \ie, seen from a view other than the object they belong to. Therefore, the model does not gain an advantage by splitting the object up into parts.

\subsection{Steering Slots via Conditioning}

    In the absence of any further supervision or control signal, splitting up an image into objects is in principle ill-defined, since we can not expect the model to settle on the desired level of granularity of what constitutes an object. For example, should a tree be treated as a single object or as separate objects (\eg, its trunk, branches and leaves)? From a practical perspective, there is no clear answer as the relevant level of detail can depend on the downstream task to solve (\eg, counting trees \vs counting how many leaves each tree has). Thus, a model should have a flexible way to adjust its internal granularity level, ideally controlled by an additional input (\ie, conditioning) signal~\citep{greffBindingProblemArtificial2020}.
    
    \paragraph{Bounding Box Conditioning} As mentioned in \cref{sec:related_work}, the SAVi architecture~\citep{kipfConditionalObjectCentricLearning2022} allows uses such conditioning signals in the form of bounding-box annotations of objects in the first frame of the videos. This lets us evaluate how well bounding-box conditioning can be used to determine the granularity level internally assumed by the model and overcome the previously observed sensitivity to the number of slots. 
    A comparison of models without additional input information (\emph{unconditioned}) with models having access to the bounding boxes of objects in the first frame (\emph{conditioned}) for the MOVi-A/C \citep{greffKubricScalableDataset2022} datasets is displayed in \cref{fig:movi_ac_savi_n_slots}.
    Here, we see that conditioning is helpful: for MOVi-A it improves the $\arecall$ without affecting the $\aprecision$ and for the more complex MOVi-C dataset, it improves both $\arecall$ and $\aprecision$.
    
    Based on these findings, we evaluate whether additional information can further improve the performance: Instead of just using the bounding-box information in the first frame to condition the model on objects, we use the bounding-box information in the last frame as a regression signal (\emph{supervised}). Note that as SAVi is trained on small consecutive video snippets, this method effectively does not require more annotated data if videos are cut such that the last and first frame of the neighboring snippets are identical.
    This approach again improves performance for both datasets (\cref{fig:movi_ac_savi_n_slots}). Specifically, it stabilizes the $\arecall$ score substantially as we increase the number of slots beyond what is required by the ground-truth annotation.

    \begin{figure}[tb]
        \centering
        \includegraphics[width=\columnwidth]{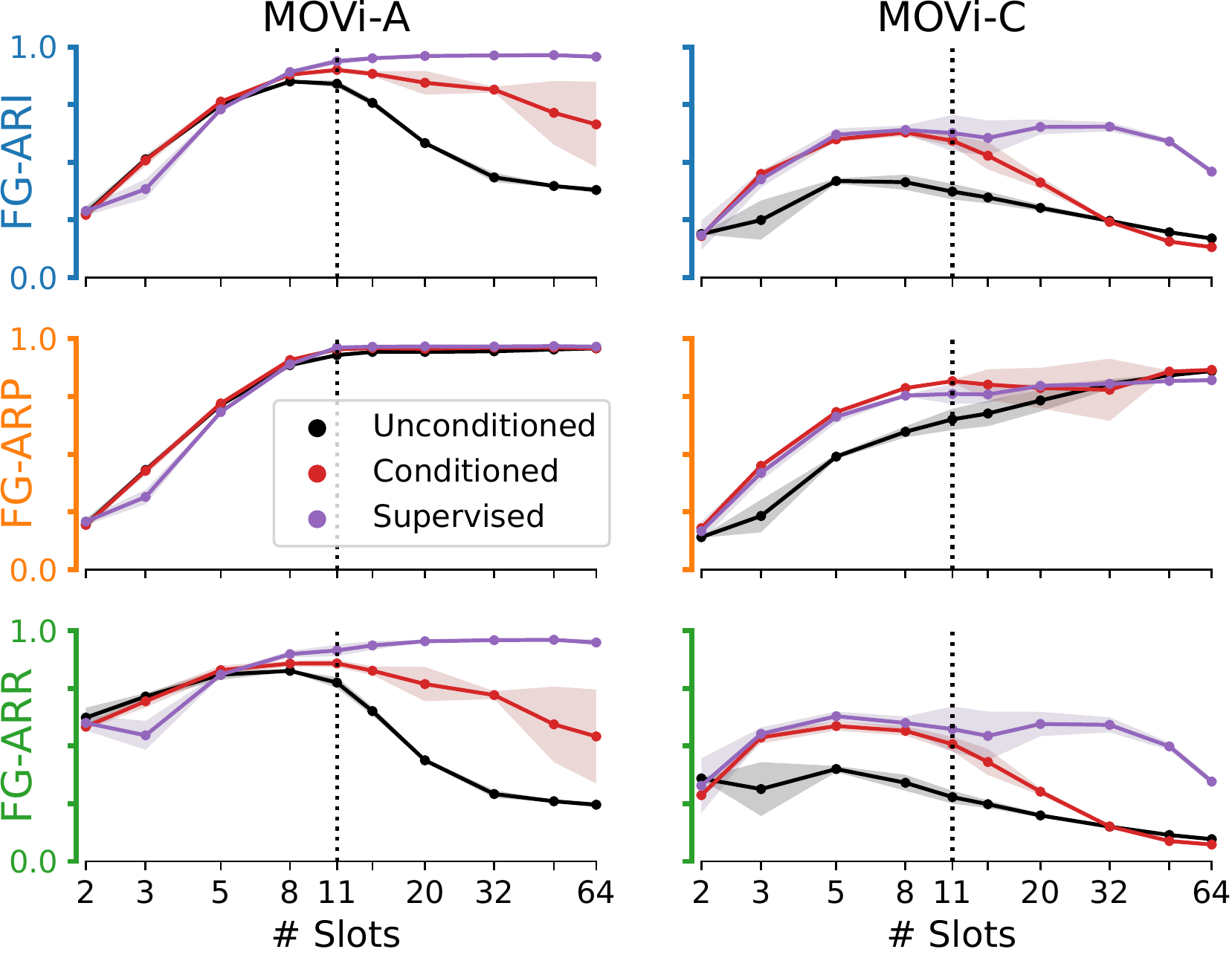}
        \caption{\textbf{SA on MOVi-A/C.} Sensitivity of Slot Attention for Video (SAVi) to the number of slots for different training variants: an unsupervised model trained on pixels only (unconditioned); an unsupervised model trained on pixels and the bounding-boxes of objects in the first frame (conditioned); the same as the conditioned model but with additional supervised bounding-box regression in the last frame.
        Scenes in both datasets contain up to 10 objects, meaning that 11 slots (dotted line) are sufficient to represent all objects and the scene's background}
        \label{fig:movi_ac_savi_n_slots}
    \end{figure}

\paragraph{Influence of Initialization of Unconditioned Slots}
    When not using any conditioning information, one still needs to initialize the slots used by Slot Attention. Common choices here are to use either a random initialization (\eg, sampling from a Gaussian) or a learned initial value \citep{locatelloObjectCentricLearningSlot2020}.
    As the previous section showed that conditioning works well to initialize slots, we now set out to test the other two initialization schemes.
    In \cref{fig:msn_e_clevr_sa_nvs_rnd_vs_learned} we compare a random with a learned initialization for SA trained with an AE and NVS objective on MSN. Noticeably, for the AE objective, there is no difference between the two initializations. For the NVS objective, however, the random initialization outperforms the learned one in all scores. Therefore, in this setting, if no conditioning information can be used, it appears to be beneficial to randomly initialize slots. For more results see~\cref{sec:extended_results}: \cref{fig:clevr_sa_nvs_rnd_vs_learned} for observations on the AE objective on CLEVR and \cref{fig:msn_e_sa_nvs_rnd_vs_learned_boqsa} for observations on a modification of slot attention on MSN-E.
    
    \begin{figure}[tbh]
        \centering
        \begin{subfigure}[b]{\linewidth}
            \includegraphics[width=\textwidth]{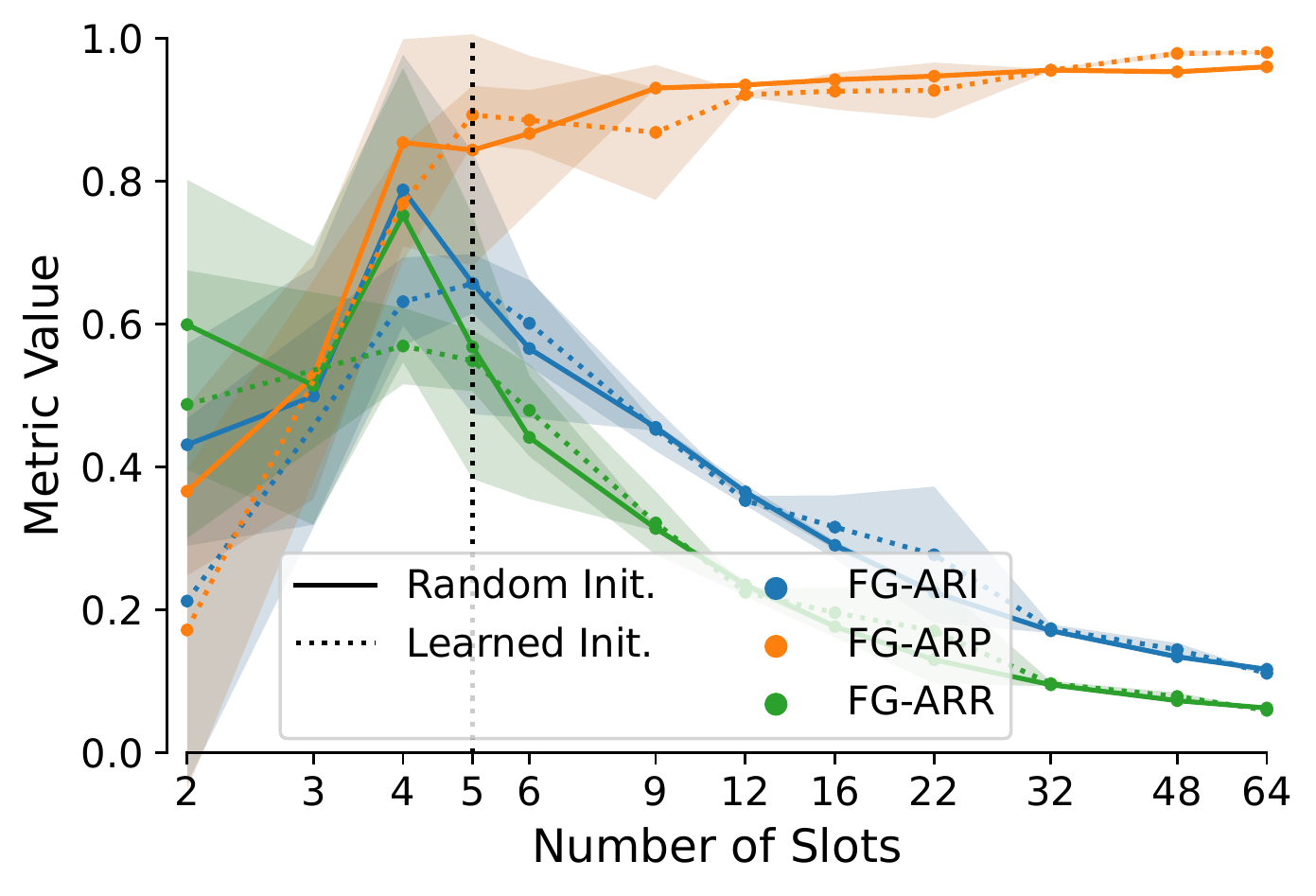}
            \caption{Slot Attention (AE) on MSN.}
            \label{fig:msn_e_sa_ae_rnd_vs_learned}
        \end{subfigure}
        \begin{subfigure}[b]{\linewidth}
            \includegraphics[width=\textwidth]{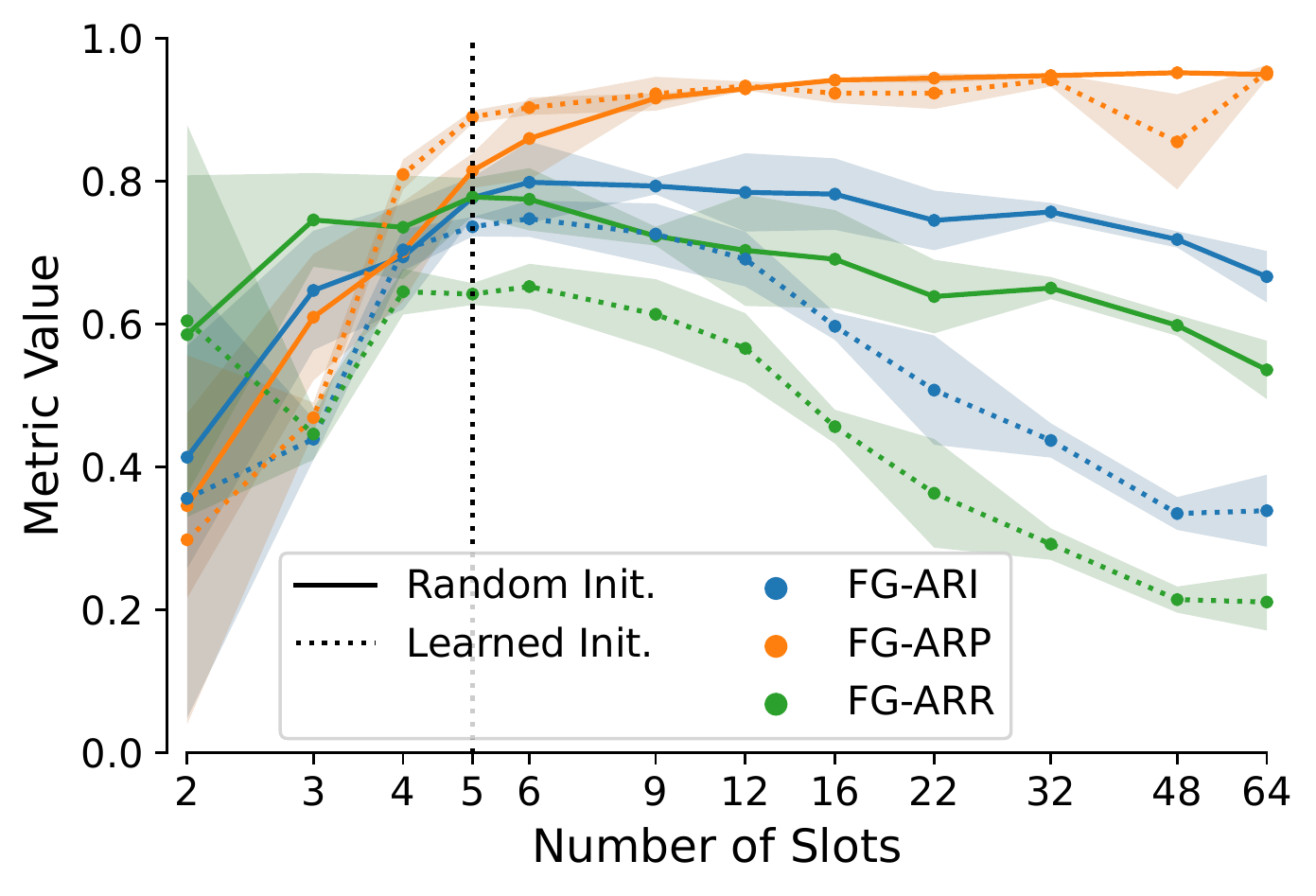}
            \caption{Slot Attention (NVS) on MSN.}
            \label{fig:msn_e_sa_nvs_rnd_vs_learned}
        \end{subfigure}
        \caption{
        \textbf{Influence of Slot Initialization.}
        Comparison of the effect of slot initialization in Slot Attention in a purely unsupervised setting without conditioning for the auto-encoding (AE) and novel-view-synthesis (NVS) objective on MSN, measured by $\fgari$ (\textcolor{ari}{blue}), $\fgaprecision$ (\textcolor{arp}{orange}), $\fgarecall$ (\textcolor{arr}{green}).
        Random initialization leads to substantially higher FG-ARR (and therefore also FG-ARI) when using too many slots for the NVS objective, while there is no notable effect for the AE objective. For results on the AE objective on CLEVR, see \cref{fig:clevr_sa_nvs_rnd_vs_learned}.
        }
        \label{fig:msn_e_clevr_sa_nvs_rnd_vs_learned}
    \end{figure}

\section{Conclusion}
    In this work, we investigated how the number of slots in slot-based object-centric models affects the learned representations and the implicit notion of objectness. To meet this end, we propose to use two new metrics, Adjusted Rand Precision and Adjusted Rand Recall, for obtaining a more detailed understanding in the behavior of object-centric models.  Specifically, we demonstrate that while PSNR or ARI are not sufficient for distinguishing failure cases (\eg, over- and undersegmentation), the proposed metrics are.
    
    Most importantly, we found that during training, adjusting this number has a crucial effect on the learned representations: If either too high or too low values are used, the model will learn to partition the scene in undesired ways. This behavior cannot be mitigated during inference by just adjusting the number of slots again. However, if the model was trained with a reasonably chosen number of slots, it remains mostly insensitive to the number of slots during inference. Despite this, we argue that this behavior is problematic as in real-world datasets the actual number of objects might vary a lot across scenes and cannot be tightly bound. 
    
    In an attempt to solve this problem, we demonstrated that by supplying the model with more information --- either in the form of conditioning information or weak supervision --- one can induce a bias toward the intended notion of objectness and partially ease the aforementioned issue.
    Finally, this work showed that different training objectives seem to create different inductive biases yielding models with different sensitivities towards the number of slots.
    
    Although all models investigated in this work are slot-based methods powered by Slot Attention, the methodology presented here is not limited to these methods. Instead, it applies to any object-centric model that either explicitly or implicitly produces instance segmentation masks.

    While we found that models trained and evaluated with too many slots use their additional slots to split up ground-truth objects into parts, it remains to be tested whether the model splits  objects into fine-grained components aligned with the human notion of objects/parts or into arbitary parts. To answer this, one could evaluate the performance of models on datasets that have segmentation annotations for different granularity levels.
    We leave this for future work.
    While this work explored the influence of model hyperparameters on the learned notion of objectness, it remains to be explored how properties of the dataset influence this. For one, this includes how the number of objects over different scenes influences the model, and for another, how independent objects need to be for the model to recognize them as separate instances.
    Finally, for practical purposes, it will be interesting to develop models allowing for a controllable notion of objectness, depending on the task at hand.

\newcommand{\RSZ}{RSZ\xspace}
\newcommand{\SvS}{SvS\xspace}
\newcommand{\KG}{KG\xspace}
\newcommand{\TK}{TK\xspace}
\newcommand{\MSMS}{MSMS\xspace}

\section*{Author Contributions}
    \RSZ led the project and was responsible for implementation, metric development, prototyping, experimentation, model analysis and writing.
    \SvS and \MSMS worked on the experimental design. \MSMS initiated the precision-recall evaluation.
    \KG co-led the project and was responsible for scoping, writing, some figures and code reviews.
    Both \KG and \TK hosted the internship resulting in this project.
    \SvS, \MSMS, \TK and \KG advised the project and reviewed the experiments' implementation.
    All authors contributed to the final version of the manuscript.

\section*{Acknowledgments}
    We thank Aravindh Mahendran, Gamaleldin Elsayed and Michael C. Mozer for insightful and stimulating discussions over the course of this project and Etienne Pot for technical support. Additionally, we thank Alexey Dosovitskiy for feedback on the manuscript.

\bibliography{references}
\bibliographystyle{icml2023}

\newpage
\appendix
\onecolumn
\section{Definition and Properties of Rand Precision and Recall} \label{appx:definition_properties_precision_recall}

    Let $\bX, \bY$ be two segmentation maps with each up to $I$ and $J$ classes, respectively. Let $m \in \mathbb{Z}^{I \times J}$ denote a matching matrix for $\bX, \bY$, i.e. $m_{ij}$ is the number of pixels that are segmented as class $i$ in $\bX$ and class $j$ in $\bY$. Further, let $m = \sum\limits_{i=1}^I\sum\limits_{j=1}^J m_{ij}$ be the total number of pixels being segmented.

    \begin{definition}[Rand Precision and Recall] \label{def:naive_precision_recall}
        Let $\bX, \bY \in \mathbb{Z}^{N}$ be two segmentation maps and let $\mathcal{S} = \{\, 0, \ldots, N] \,\} $ denote the set of all pixel indices. The rand precision ($\precision: \mathbb{Z}^{N} \times \mathbb{Z}^{N} \to [0, 1]$) and rand recall ($\recall: \mathbb{Z}^{N} \times \mathbb{Z}^{N} \to [0, 1]$) are defined as:
        \begin{align} 
            \precision(\bX, \bY) &= \frac{\sum\limits_{i, j \in \Ss} \delta_{\bX_i, \bX_j} \delta_{\bY_i, \bY_j}}{\sum\limits_{i, j \in \Ss} \delta_{\bX_i, \bX_j} \delta_{\bY_i, \bY_j} + \sum\limits_{i, j \in \Ss} (1 - \delta_{\bX_i, \bX_j}) \delta_{\bY_i, \bY_j}} \label{eq:naive_def_precision} \\
            \recall(\bX, \bY) &= \frac{\sum\limits_{i, j \in \Ss} \delta_{\bX_i, \bX_j} \delta_{\bY_i, \bY_j}}{\sum\limits_{i, j \in \Ss} \delta_{\bX_i, \bX_j} \delta_{\bY_i, \bY_j} + \sum\limits_{i, j \in \Ss} \delta_{\bX_i, \bX_j} (1 - \delta_{\bY_i, \bY_j})} \label{eq:naive_def_recall}.
        \end{align}
        While the rand precision $\precision(\bX, \bY)$ measures for how many (of all possible) tuples of pixels for which $\bY$ is identical, $\bX$ is also identical, the rand recall $\recall(\bX, \bY)$ measures for how many (of all possible) tuples of pixels for which $\bX$ is identical, $\bY$ is also identical.
    \end{definition}
    
    Note that these metrics have been proposed before in the context of measuring the general similarity of clusterings \citep{wallaceMethodComparingTwo1983}. Here, however, we propose to leverage both these metrics for getting a more fine-grained insight into the behavior of segmentation models: Do they perform undersegmentation and merge (parts of) unrelated objects or do they perform oversegmentation and split objects up in parts?

    To make the precision and recall metrics more comparable to the commonly used Adjusted Rand Index (ARI), we need to adjust them, too: We correct them for chance agreement, such that we normalize the metrics by the values a random segmentation would yield. To do this, we first introduce an alternative but equal way to define the two new metrics:

    \begin{definition}[Alternative Definition of Rand Precision and Recall.]  \label{def:simple_precision_recall}
        Let $\bX, \bY$ be two segmentation maps with each up to $I$ and $J$ classes, respectively. Let $m \in \mathbb{Z}^{I \times J}$ denote a matching matrix for $\bX, \bY$, i.e. $m_{ij}$ is the number of pixels that are segmented as class $i$ in $\bX$ and class $j$ in $\bY$. Further, let $m = \sum_{i=1}^I\sum_{j=1}^J m_{ij}$ be the total number of pixels being segmented.
        We define the rand precision ($\precision: \mathbb{Z}^{N} \times \mathbb{Z}^{N} \to [0, 1]$) and rand recall ($\recall: \mathbb{Z}^{N} \times \mathbb{Z}^{N} \to [0, 1]$) as:
        \begin{align}
            \precision(\bX, \bY) &= \alpha_\textrm{Precision} + \beta_\textrm{Precision} \sum\limits_{i=1}^I\sum\limits_{j=1}^J m_{ij}^2 \textrm{ with } \alpha_\textrm{Precision} = -\frac{m}{Q}, \beta_\textrm{Precision} = \frac{1}{Q}, \label{eq:simple_def_precision} \\
            \recall(\bX, \bY) &= \alpha_\textrm{Recall} + \beta_\textrm{Recall} \sum\limits_{i=1}^I\sum\limits_{j=1}^J m_{ij}^2 \textrm{ with } \alpha_\textrm{Recall} = -\frac{m}{P}, \beta_\textrm{Recall} = \frac{1}{P} \label{eq:simple_def_recall} ,
        \end{align}
        with $P = \sum_{i=1}^I m_{i+}^2 - m = \sum_{i=1}^I (\sum_{j=1}^J m_{ij})^2 - m$ and $Q = \sum_{j=1}^J m_{+j}^2 - m = \sum_{j=1}^J (\sum_{i=1}^I m_{ij})^2 - m$. 
    \end{definition}
    
    \begin{prop}[Equality of \cref{def:naive_precision_recall} and \cref{def:simple_precision_recall}]
        The definitions \cref{def:naive_precision_recall} and \cref{def:simple_precision_recall} are consistent and define the same precision and recall functions.
    \end{prop}
    \begin{proof}
        According to \citet{albatineh2006similarity} the following relations hold:
        \begin{align}
            \sum\limits_{i, j \in \Ss} \delta_{\bX_i, \bX_j} \delta_{\bY_i, \bY_j} &= \frac{1}{2}\sum\limits_{i=1}^I\sum\limits_{j=1}^J m_{ij}^2 - \frac{m}{2}, \\
            \sum\limits_{i, j \in \Ss} \delta_{\bX_i, \bX_j} (1 - \delta_{\bY_i, \bY_j}) &= \frac{P}{2} - \frac{1}{2}\sum\limits_{i=1}^I\sum\limits_{j=1}^J m_{ij}^2 + \frac{m}{2}, \\
            \sum\limits_{i, j \in \Ss} (1 - \delta_{\bX_i, \bX_j}) \delta_{\bY_i, \bY_j} &= \frac{Q}{2} - \frac{1}{2}\sum\limits_{i=1}^I\sum\limits_{j=1}^J m_{ij}^2 + \frac{m}{2}.
        \end{align}
        Inserting this in \cref{eq:naive_def_precision} yields
        \begin{align}
             \precision(\bX, \bY) &= \frac{\frac{1}{2}\sum\limits_{i=1}^I\sum\limits_{j=1}^J m_{ij}^2 - \frac{m}{2}}{\frac{1}{2}\sum\limits_{i=1}^I\sum\limits_{j=1}^J m_{ij}^2 - \frac{m}{2} + \frac{P}{2} - \frac{1}{2}\sum\limits_{i=1}^I\sum\limits_{j=1}^J m_{ij}^2 + \frac{m}{2}} \\
             &= \frac{\frac{1}{2}\sum\limits_{i=1}^I\sum\limits_{j=1}^J m_{ij}^2 - \frac{m}{2}}{\frac{Q}{2}}
             = \frac{\sum\limits_{i=1}^I\sum\limits_{j=1}^J m_{ij}^2 - m}{Q} \\
             &= -\frac{m}{Q} + \frac{1}{Q} \sum\limits_{i=1}^I\sum\limits_{j=1}^J m_{ij}^2,
        \end{align}
        recovering \cref{eq:simple_def_precision}. The equality of \cref{eq:naive_def_recall} and \cref{eq:simple_def_recall} follows analogously, concluding the proof.
    \end{proof}
    
    After establishing equality of the definition, we will continue using the less intuitive but mathematically more convenient \cref{def:simple_precision_recall}. Now we can continue with normalizing the metrics. For this, we use the result of \citet{albatineh2006similarity}, that any similarity index in the form $S = \alpha_\textrm{S} + \beta_\textrm{S} \sum_{i=1}^I\sum_{j=1}^J m_{ij}^2$ can be adjusted by 
    \begin{align} \label{eq:adjusted_similarity_index}
        \operatorname{AS} = \frac{\operatorname{S}(\bX, \bY) - \E_{\bY'}\left[\operatorname{S}(\bX, \bY')\right]}{1 - \E_{\bY'}\left[\operatorname{S}(\bX, \bY')\right]} = \frac{\sum\limits_{i=1}^I\sum\limits_{j=1}^J m_{ij}^2 - \E_{\bY'}\left[\sum\limits_{i=1}^I\sum\limits_{j=1}^J m_{ij}^2\right]}{\frac{1 - \alpha_\textrm{S}}{\beta_\textrm{S}} - \E_{\bY'}\left[\sum\limits_{i=1}^I\sum\limits_{j=1}^J m_{ij}^2\right]},
    \end{align}
    where $\E_{\bY'}[\sum_{i=1}^I\sum_{j=1}^J m_{ij}^2]$ can be computed using a generalized hypergeometric distribution \citep{hubertComparingPartitions1985} as
    \begin{align}
        \E_{\bY'}\left[\sum\limits_{i=1}^I\sum\limits_{j=1}^J m_{ij}^2\right] = \frac{\sum\limits_{i=1}^I\sum\limits_{j=1}^J m_{i+}^2 m_{+j}^2}{m(m-1)} + \frac{m^2 - \left( \sum\limits_{i=1}^I m_{i+}^2 \sum\limits_{j=1}^J m_{+j}^2 \right)}{m-1}.
    \end{align}
    We denote the adjusted precision and adjusted recall as $\aprecision$ and $\arecall$, respectively:
    \begin{definition}[Adjusted Rand Precision \& Recall.] \label{def:adjusted_precision_recall}
        \begin{align}
            \aprecision(\bX, \bY) &= \frac{\sum\limits_{i=1}^I\sum\limits_{j=1}^J m_{ij}^2 - \E_{\bY'}\left[\sum\limits_{i=1}^I\sum\limits_{j=1}^J m_{ij}^2\right]}{\frac{1 - \alpha_\textrm{Precision}}{\beta_\textrm{Precision}} - \E_{\bY'}\left[\sum\limits_{i=1}^I\sum\limits_{j=1}^J m_{ij}^2\right]} \\
            \arecall(\bX, \bY) &= \frac{\sum\limits_{i=1}^I\sum\limits_{j=1}^J m_{ij}^2 - \E_{\bY'}\left[\sum\limits_{i=1}^I\sum\limits_{j=1}^J m_{ij}^2\right]}{\frac{1 - \alpha_\textrm{Recall}}{\beta_\textrm{Recall}} - \E_{\bY'}\left[\sum\limits_{i=1}^I\sum\limits_{j=1}^J m_{ij}^2\right]}.
        \end{align}
    \end{definition}
    An implementation of these metrics is given in \cref{lst:arr_arp}.
    
    \begin{prop}[Symmetry of arguments.] \label{prop:ar_ap_antisymmetry}
        For two segmentation maps $\bX, \bY \in \mathbb{Z}^{H \times W}$, $\precision(\bX, \bY) = \recall(\bY, \bX)$.
    \end{prop}
    \begin{proof}
        Recalling the first definition of the rand precision and recall \cref{def:naive_precision_recall} yields
        \begin{align}
            \recall(\bY, \bX) &= \frac{\sum\limits_{i, j \in \Ss} \delta_{\bY_i, \bY_j} \delta_{\bX_i, \bX_j}}{\sum\limits_{i, j \in \Ss} \delta_{\bY_i, \bY_j} \delta_{\bX_i, \bX_j} + \sum\limits_{i, j \in \Ss} \delta_{\bY_i, \bY_j} (1 - \delta_{\bX_i, \bX_j})} \\
            &= \frac{\sum\limits_{i, j \in \Ss} \delta_{\bX_i, \bX_j} \delta_{\bY_i, \bY_j}}{\sum\limits_{i, j \in \Ss} \delta_{\bX_i, \bX_j} \delta_{\bY_i, \bY_j} + \sum\limits_{i, j \in \Ss} (1 - \delta_{\bX_i, \bX_j}) \delta_{\bY_i, \bY_j}} \\
            &= \precision(\bX, \bY),
        \end{align}
        concluding the proof.
    \end{proof}
    
    \begin{prop}[ARI is the $F_1$ score of ARP and ARR] \label{prop:ari_f1_harmonic_mean_ap_ar}
        The $\ari$ is the $F_1$ score (i.e., harmonic mean) of the $\aprecision$ and $\arecall$:
        \begin{align}
            \ari = \frac{2}{\aprecision^{-1} + \arecall^{-1}}.
        \end{align}
    \end{prop}
    \begin{proof}
        We begin by rewriting \cref{eq:adjusted_similarity_index} as
        \begin{align}
            \operatorname{AS} = \frac{A}{\gamma - B}, 
        \end{align}
        with $\gamma = \frac{1 - \alpha}{\beta}$, $A = \sum\limits_{i=1}^I\sum\limits_{j=1}^J m_{ij}^2 - \E_{\bY'}[\sum\limits_{i=1}^I\sum\limits_{j=1}^J m_{ij}^2]$ and $B = \E_{\bY'}[\sum\limits_{i=1}^I\sum\limits_{j=1}^J m_{ij}^2]$ to ease notation. Note that for the similarity indices in question - ARI, AP and AR - $\gamma$ has the values $\gamma_\textrm{RI} = \frac{P + Q}{2} + m, \gamma_\textrm{Precision} = Q + m$ and $\gamma_\textrm{Recall} = P + m$ and $A, B$ are shared across them.
        
        We continue by showing a relation between the three $\gamma$ values:
        \begin{align}
            \gamma_\textrm{RI} &= \frac{\gamma_\textrm{Precision} + \gamma_\textrm{Recall}}{2}
            = \frac{1}{2} \left( \frac{A}{\aprecision} + B + \frac{A}{\arecall} + B \right)
            = \frac{A}{2} \left( \frac{1}{\aprecision} + \frac{1}{\arecall} \right) + B
        \end{align}
        Inserting this relation into the definition of $\ari$ yields
        \begin{align}
            \ari &= \frac{A}{\gamma_\textrm{RI} - B}
            = \frac{A}{\frac{A}{2}(\frac{1}{\aprecision} + \frac{1}{\arecall})}
            = \frac{2}{\aprecision^{-1} + \arecall^{-1}},
        \end{align}
        concluding the proof.
    \end{proof}
    
    \begin{prop}[$\aprecision$ and $\arecall$ bound $\ari$.] \label{prop:ap_ar_bound_ari}
        For two segmentation maps $\bX, \bY \in \mathbb{Z}^{H \times W}$: 
        \begin{align}
            \min(\aprecision(\bX, \bY), \arecall(\bX, \bY)) \leq \ari(\bX, \bY) \leq \max(\aprecision(\bX, \bY), \arecall(\bX, \bY)).
        \end{align}
    \end{prop}
    \begin{proof}
        Per \cref{prop:ari_f1_harmonic_mean_ap_ar} the $\ari$ is the harmonic mean of $\aprecision$ and $\arecall$, i.e.
        \begin{align}
            \ari^{-1} = \frac{\aprecision^{-1} + \aprecision^{-1}}{2}.
        \end{align}
        From here it directly follows that $\ari^{-1}$ and, therefore, $\ari$ is bound by $\aprecision$ and $\arecall$, too.
    \end{proof}
    
    Note that although the previous proposition only holds for \emph{single} pairs of segmentation maps (\ie, single dataset samples), we empirically find that the relation also almost alaways holds when averaging $\ari$, $\aprecision$ and $\arecall$ over multiple samples/a full dataset.

\section{Experimental Details} \label{sec:experiemntal_details}
    All experimental results presented in this paper are averages (and standard deviations, if applicable) over three random seeds, with the exception of the experiments with OSRT. Here, for computational reasons, results for a single seed are shown.

    \paragraph{Vanilla Slot Attention (AE)}
    We use the same general architecture as proposed by \citet{locatelloObjectCentricLearningSlot2020}, namely a CNN encoder, followed by Slot Attention and finally by a CNN decoder.
    For CLEVR, we used the same encoder architecture as \citet{locatelloObjectCentricLearningSlot2020}. This is followed by Slot Attention with $3$ iterations, a query/key/value projection dimensionality of $128$ and hidden MLP dimensionality of $256$. Finally, as a decoder we use a spatial broadcast decoder consisting of four transposed convolutional layers and linear position embedding \citep{watters2019spatial,locatelloObjectCentricLearningSlot2020}. 
    For MSN we use a ResNet-18 \citep{heDeepResidualLearning2015} encoder, and reduce the query/key/value projection and hidden MLP dimensionality to $64$ and $128$, respectively. The decoder now consists of 4 transposed and 1 normal convolutional layer and linear position embedding.
    We train models on both datasets for $300,000$ iterations with a batch size of $64$ using Adam \citep{kingmaAdamMethodStochastic2017}.
    
    \paragraph{Slot Attention for Video}
    We use the same architecture as proposed by \citet{kipfConditionalObjectCentricLearning2022}. Specifically, the encoder consists of five convolutional layers with kernel size $5$, followed by a linear position embedding layer \cite{locatelloObjectCentricLearningSlot2020}, layer norm and two additional convolutional layers with a kernel size of $1$. All convolutional layers except the last are followed by a ReLU non-linearity. While for the higher-dimensional MOVi images the convolutional layers use a stride of $2$, for CATER a stride of $1$ is used.
    The encoder is followed by a ``corrector'' represented by Slot Attention with $2$ iterations, a query/key/value projection dimensionality of $128$ and a hidden MLP dimension of $256$.
    Following \citet{kipfConditionalObjectCentricLearning2022} we used no predictor for experiments on CATER and a multi-head dot-product attention mechanism \citep{vaswaniAttentionAllYou2017} with $4$ attention heads, a projection dimensionality of $128$ and hidden MLP dimensionality of $256$. 
    Finally, as a decoder the same spatial broadcast decoder architecture~\citep{watters2019spatial} as used by \citet{locatelloObjectCentricLearningSlot2020} was used.
    We train the models for $200,000$ iterations with Adam with a batch size of $64$.
    
    \paragraph{Vanilla Slot Attention (NVS)}
    We use the same encoder architecture as outline above for the AE objective, expect that we introduce an additional camera-pose embedding, following \citet{sajjadiSceneRepresentationTransformer2022}. To save memory, we reduce the number of Slot Attention iterations down to $1$. For the decoder, we also use the same architecture as above, except that we again introduce a camera-pose embedding along the linear position embedding encoding.
    We train the models for $500,000$ iterations with Adam with a batch size of $64$.
    
    \paragraph{Object Scene Representation Transformer}
    We use the same architecture and hyperparameters as proposed by \citet{sajjadiObjectSceneRepresentation2022} with the exception that instead of using a learned initialization of the slots in the Slot Attention module, we randomly initialize them.
    
\section{Extended Version of \cref{fig:precsion_recall_ari_checkerboard}} \label{sec:precsion_recall_ari_checkerboard_details}
    A more detailed version of \cref{fig:extended_precsion_recall_ari_checkerboard} is displayed in \ref{fig:extended_precsion_recall_ari_checkerboard}. Specifically, the $\ari$, $\aprecision$ and $\recall$ curves shown here were generated on the synthethic data shown on the left side of the figure. Namely, the ground-truth equals a square divided into $4 \times 4$ different segments. In the predictions, we simulated on the one hand merging of different segments (undersegmentation) and on the other separation of segments (oversegmentation). For visualization purposes, we only show a zoomed-in version of the simulated predictions.
    \begin{figure}[tbh]
        \centering
        \includegraphics[width=\linewidth]{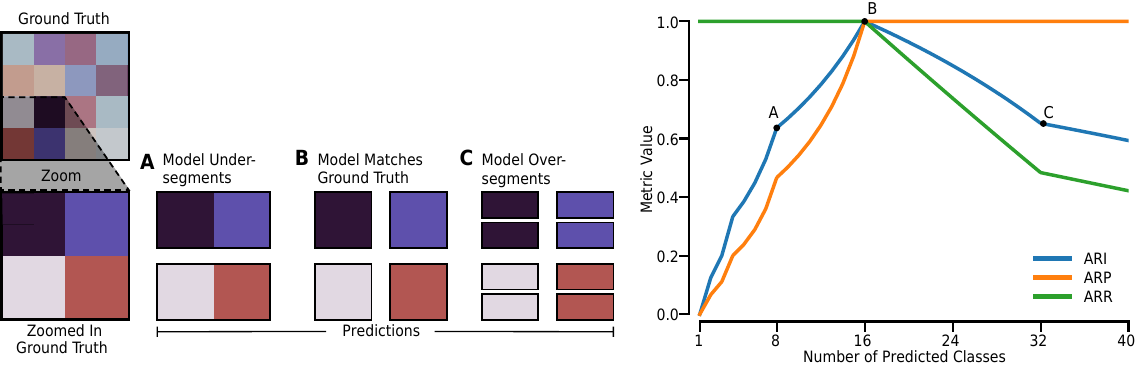}
        \caption{\textbf{Zooming into imperfect segmentations.}
        Detailed version of \cref{fig:precsion_recall_ari_checkerboard}.
        Left: Ground-truth segmentation of a zoomed-in image patch and predictions for three models with different characteristics.
        Right: $\ari$, $\aprecision$ and $\arecall$ as a function of number of classes in the simulated predictions.
        $\ari$ is clearly sensitive to both under- and oversegmentation, however it does not distinguish between these cases, as models A and C achieve the same $\ari$ score.
        In contrast, the proposed $\aprecision$ and $\arecall$ metrics successfully differentiate the models: the undersegmenting model A has a low $\aprecision$ but perfect $\arecall$, while the oversegmenting model B has perfect $\aprecision$, but low $\arecall$. The number of predicted classes has been changed by starting from the ground-truth segmentation (\ie, 16 classes), and iteratively merging neighboring segments or splitting segments to increase and decrease the number of classes, respectively. \ms{consider a longer explanation as to how the plot has been generated (\ie, how are we adding segments as we go from 1 to 40?}
        }
        \label{fig:extended_precsion_recall_ari_checkerboard}
    \end{figure}
    
\clearpage
\section{Extended Experimental Results} \label{sec:extended_results}
    \paragraph{Influence of Slot Initialization} Analogously to \cref{fig:msn_e_clevr_sa_nvs_rnd_vs_learned}, we tested the influence of the slot initialization --- random \vs a learned initialization --- for CLEVR in \cref{fig:clevr_sa_nvs_rnd_vs_learned} and obtain similar results.
    \begin{figure}[h!]
        \centering
        \includegraphics[width=0.5\linewidth]{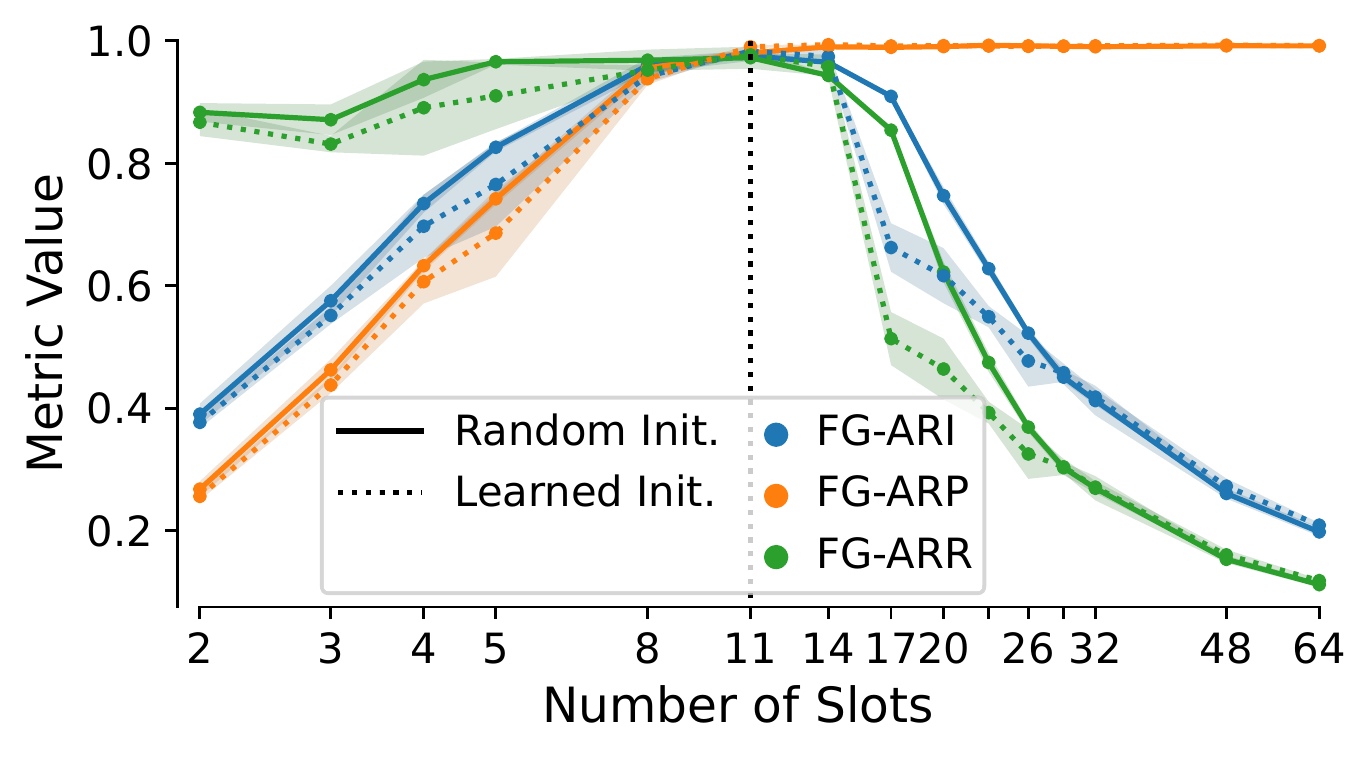}
        \caption{
        \textbf{Influence of Slot Initialization for Slot Attention (AE) on CLEVR.}
        Comparison of the effect of slot initialization in Slot Attention in a purely unsupervised setting without conditioning for the auto-encoding (AE) objective on CLEVR, measured by $\fgari$ (\textcolor{ari}{blue}), $\fgaprecision$ (\textcolor{arp}{orange}), $\fgarecall$ (\textcolor{arr}{green}).
        Random initialization leads to higher FG-ARR (and therefore also FG-ARI) when using too many slots for the NVS objective.
        }
        \label{fig:clevr_sa_nvs_rnd_vs_learned}
    \end{figure}

    Furthermore, we also test the influence of the slot initialization for another variant of Slot Attention: Bi-level Optimization for Slot Attention (BO-QSA) \citep{jiaimproving}. This model leverages the implicit optimization technique proposed by~\citet{chang2022object} to use learned initializations of slots and is claimed to outperform various earlier slot-based models in unsupervised image segmentation and reconstruction. We now test whether using implicit optimization does indeed solve the previously observed issue with learned initializations (see~\cref{fig:msn_e_sa_nvs_rnd_vs_learned}) using a re-implementation of the model using the NVS objective on MSN-E. The results in~\cref{fig:msn_e_sa_nvs_rnd_vs_learned_boqsa} clearly show --- in line with our previous results --- that random initialization outperforms learned initialization of slots also when for BO-QSA.
    \begin{figure}[h!]
        \centering
        \includegraphics[width=0.48\linewidth]{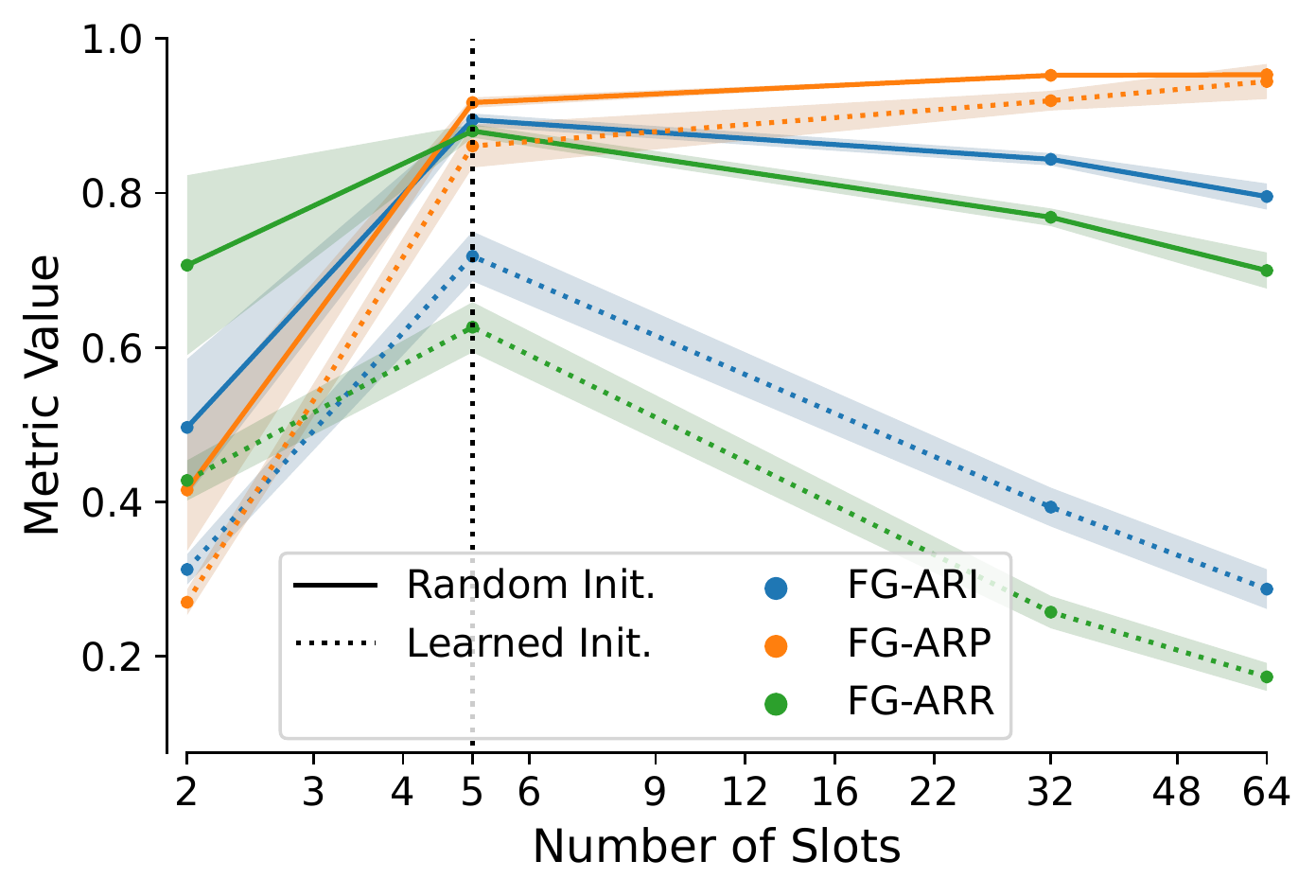}
        \caption{
        \textbf{Influence of Slot Initialization for BO-QSA (NVS) on MSN-E.}
        Comparison of the effect of slot initialization in Bi-level Optimized Query Slot Attention (BO-QSA) in a purely unsupervised setting without conditioning for novel-view-synthesis (NVS) objective on MSN-E, measured by $\fgari$ (\textcolor{ari}{blue}), $\fgaprecision$ (\textcolor{arp}{orange}), $\fgarecall$ (\textcolor{arr}{green}).
        As for normal SA (see~\cref{fig:msn_e_sa_nvs_rnd_vs_learned}) random initialization leads again to higher FG-ARR (and therefore also FG-ARI) when using too many slots for the NVS objective.
        }
        \label{fig:msn_e_sa_nvs_rnd_vs_learned_boqsa}
    \end{figure}
    
\paragraph{Slot Attention for Video on CATER} Analogously to \cref{fig:movi_ac_savi_n_slots}, we also tested the influence of conditioning on SAVi on the CATER dataset \cite{girdharCATERDiagnosticDataset2020}. The results, displayed in \cref{fig:cater_savi_n_slots}, indicate the same model behavior as the results for MOVi-A/C.
    \begin{figure}[ht!]
        \centering
        \includegraphics[width=\linewidth]{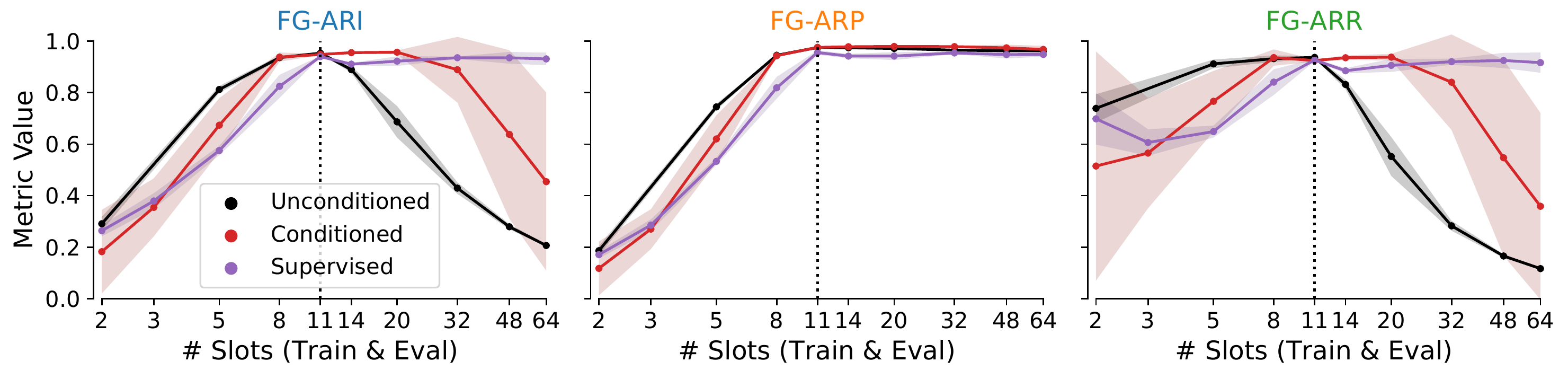}
        \caption{\textbf{CATER}. Sensitivity of Slot Attention for Video (SAVi) to the number of slots for different training variants: an unsupervised model trained on pixels only (unconditioned); an unsupervised model train on pixels and the bounding-boxes of objects in the first frame (conditioned); the same as the conditioned model but with additional supervised bounding-box regression in the last frame.
        Scenes contain up to 10 objects, meaning that 11 slots (dotted line) are sufficient to represent all objects and the scene's background.}
        \label{fig:cater_savi_n_slots}
    \end{figure}
    
\paragraph{Dependence of Model Performance on Number of Objects} \label{sec:model_performance_number_of_objects}
    The results in \cref{sec:experiments}, most importantly in \cref{fig:movi_ac_savi_n_slots}, showed that models yield undesired object partitions if they are trained with too few slots, when evaluated over the full test dataset. However, there are two possible behaviors yielding to the same observation: For one, the model could just yield suboptimal segmentations for all samples independent of the actual number of objects; for another, it could yield the desired partition for samples with few objects and become gradually worse the more objects are present in the scene. While the former option corresponds to a total failure of the model in terms of object discovery, the latter, while imperfect, would be better.
    \begin{figure}[h!]
        \centering
        \includegraphics[width=0.6\textwidth]{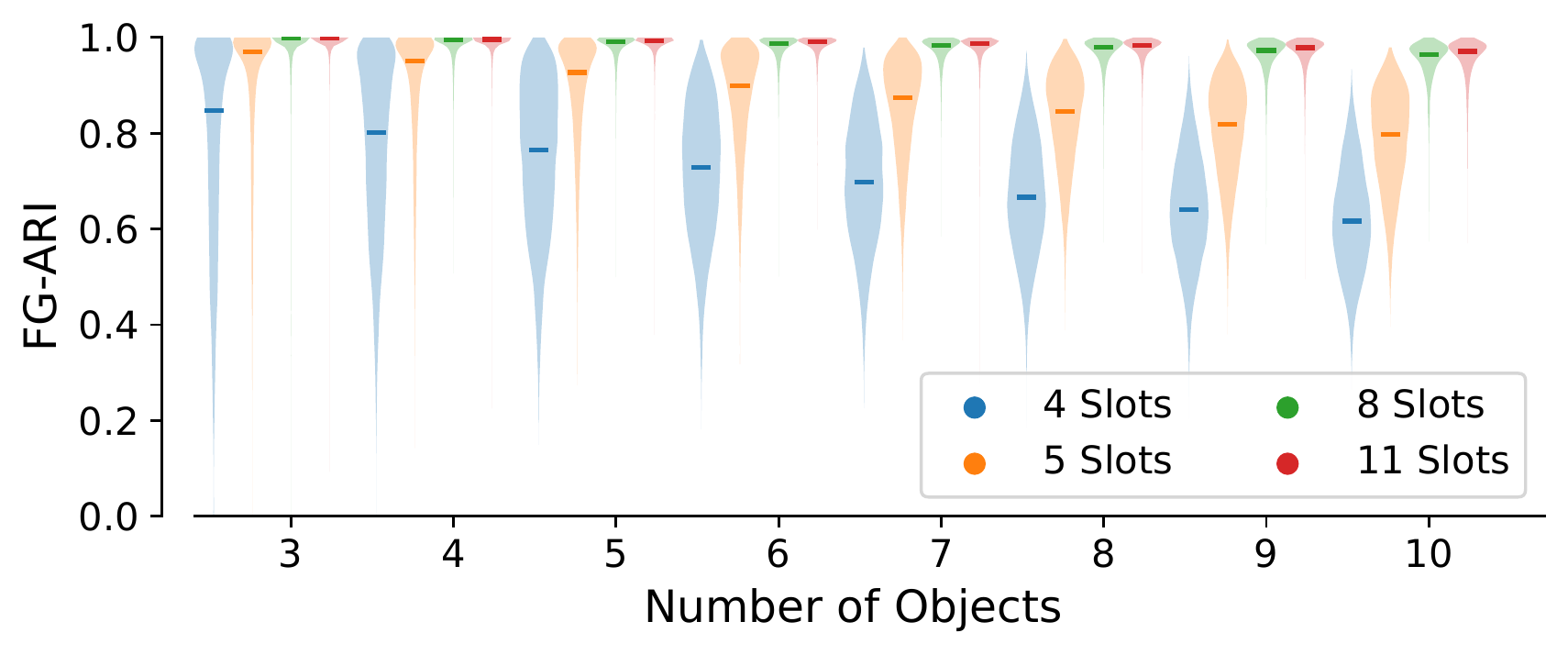}
        \caption{\textbf{MOVi-A}. Segmentation performance measured by FG-ARI as a function of number of objects in the test samples for models with different numbers of slots.}
        \label{fig:movi_a_ari_per_number_of_objects}
    \end{figure}
    
    To investigate this, we visualize the performance of unconditioned SAVi models with a varying number of slots on MOVi-A in \Cref{fig:movi_a_ari_per_number_of_objects}, grouped by the number of ground-truth objects per scene. Here, we see again that the more slots the model has, the better the overall performance becomes. Moreover, the higher the number of slots is, the lower is the performance decrease observed for samples with an increasing number of objects. Note, that the model with the fewest slots (\ie four) underperforms even compared to the other models for scenes with only three objects. This shows that too low number of slots during training induces an undesired notion of objectness resulting in undesired segmentation masks even for samples that could be represented with that number of slots.

\clearpage
\section{Implementation of Metrics}
\begin{lstlisting}[language=Python, caption={PyTorch implementation of the suggested (Adjusted) Rand Precision/Recall}, captionpos=b, label=lst:arr_arp]
import torch
import torch.nn.functional as F

def precision_recall(segmentation_gt: torch.Tensor, segmentation_pred: torch.Tensor,
  mode: str, adjusted: bool):
  """ Compute the (Adjusted) Rand Precision/Recall.

  Args:
    segmentation_gt: Int tensor with shape (batch_size, height, width) containing the
      ground-truth segmentations.
    segmentation_pred: Int tensor with shape (batch_size, height, width) containing the
      predicted segmentations.
    mode: Either "precision" or "recall" depending on which metric shall be computed.
    adjusted: Return values for adjusted or non-adjusted metric.
  Returns:
    Float tensor with shape (batch_size), containing the (Adjusted) Rand
      Precision/Recall per sample.
  """
  max_classes = max(segmentation_gt.max(), segmentation_pred.max()) + 1
  oh_segmentation_gt = F.one_hot(segmentation_gt, max_classes)
  oh_segmentation_pred = F.one_hot(segmentation_pred, max_classes)

  coincidence = torch.einsum("bhwk,bhwc->bkc", oh_segmentation_gt, oh_segmentation_pred)
  coincidence_gt = coincidence.sum(-1)
  coincidence_pred = coincidence.sum(-2)

  m_squared = torch.sum(coincidence**2, (1, 2))
  m = torch.sum(coincidence, (1, 2))
  # How many pairs of pixels have the smae label assigned in ground-truth segmentation.
  P = torch.sum(coincidence_gt * (coincidence_gt - 1), -1)
  # How many pairs of pixels have the smae label assigned in predicted segmentation.
  Q = torch.sum(coincidence_pred * (coincidence_pred - 1), -1)

  expected_m_squared = (P + m) * (Q + m) / (m * (m - 2)) + (m**2 - Q - P -2 * m) / (m - 1)

  if mode == "precision":
    gamma = P + m
  elif mode == "recall":
    gamma = Q + m
  else:
    raise ValueError("Invalid mode.")
  if adjusted:
    return (m_squared - expected_m_squared) / (gamma - expected_m_squared)
  else:
    return (m_squared - m) / (gamma - m)
\end{lstlisting}

\end{document}